\pdfoutput=1
\def\year{2022}\relax
\documentclass[letterpaper]{article} %
\usepackage{aaai22}  %
\usepackage{times}  %
\usepackage{helvet}  %
\usepackage{courier}  %
\usepackage[hyphens]{url}  %
\usepackage{graphicx} %
\usepackage{natbib}  %
\usepackage{caption} %
\DeclareCaptionStyle{ruled}{labelfont=normalfont,labelsep=colon,strut=off} %
\usepackage{algorithm}
\usepackage{algorithmic}

\usepackage{microtype}
\usepackage{amsmath}    %
\usepackage{amssymb}    %
\usepackage{amsthm}     %

\usepackage{bm}         %
\usepackage{thmtools}
\usepackage{thm-restate}
\usepackage{hyperref}
\usepackage[capitalise]{cleveref}

\newtheorem{theorem}{Theorem}
\newtheorem{lemma}[theorem]{Lemma}

\theoremstyle{definition}
\newtheorem{definition}[theorem]{Definition}
\newtheorem*{terminology*}{Terminology}

\usepackage{subcaption}

\usepackage[disable]{todonotes}
\newcommand{\Todo}[2][]{\todo[inline, #1]{#2}}
\usepackage{my/symbols}
\usepackage{sty/more-symbols}
\usepackage[decisionutilitycolor,rectangularnodes]{my/influence-diagrams}
\usepackage{mathtools}
\usepackage{sty/chrispvm-math-for-mdi}
\usepackage{subcaption}
\usepackage{wrapfig}

\TaskifyTwoDifffalse 
\TaskifyTwoVersiontrue

\usepackage{savesym}
\usepackage{changes}
\definechangesauthor[color=olive]{Chris}
\definechangesauthor[color=olive]{chris}
\definechangesauthor[color=blue]{Ryan}
\definechangesauthor[color=blue]{ryan}
\savesymbol{comment} %
\RequirePackage{comment}

\usepackage{newfloat}
\usepackage{listings}
\floatstyle{ruled}
\newfloat{listing}{tb}{lst}{}
\floatname{listing}{Listing}
\title{A Complete Criterion for Value of Information in Soluble Influence Diagrams}
\author{
    Chris van Merwijk*,\textsuperscript{\rm 1}
    Ryan Carey*,\textsuperscript{\rm 1}
    Tom Everitt\textsuperscript{\rm 2}
}
\affiliations{
    *Equal contribution,
    \textsuperscript{\rm 1}University of Oxford,
    \textsuperscript{\rm 2}DeepMind \\
    chris.vanmerwijk@philosophy.ox.ac.uk, ryan.carey@jesus.ox.ac.uk, tomeveritt@google.com
    
}

\begin{document}

\maketitle

\begin{abstract}
Influence diagrams have recently been used to analyse the safety and fairness properties of AI systems.
A key building block for this analysis is a graphical criterion for value of information (VoI).
This paper establishes the first complete graphical criterion for VoI in influence diagrams with multiple decisions.
Along the way, we establish two important techniques for proving properties of multi-decision influence diagrams:
ID homomorphisms are structure-preserving transformations of influence diagrams, while a Tree of Systems is a collection of paths that captures how information and control can flow in an influence diagram. %

\end{abstract}

\section{Introduction}

One approach to analysing the safety and fairness of AI systems is to represent them using variants of Bayesian networks \citep{Everitt2019modeling, Kusner2017}.
Influence diagrams (IDs) can be viewed an extension of Bayesian networks for representing agents \citep{Howard2005,Everitt2021agent}.
This graphical perspective offers a concise view of key relationships, that abstracts away from much of the internal complexity of modern-day AI systems.

Once a decision problem is represented graphically, key aspects can be summarised.
One well-studied concept is the \emph{value of information} (VoI) \citep{Howard1966}, which describes how much more utility 
an agent is able to obtain if it can observe a variable in its environment, compared with if it cannot.
Other summary concepts includes ``materiality'', ``value of control'', 
``response incentives''.~\looseness=-1

These concepts have been used to analyse the
redirectability \citep{everitt2019tampering,Holtman2020} of AI systems,
fairness \citep{Everitt2021agent,Ashurst2022}, 
ambitiousness \citep{cohen2020unambitious},
and %
the safety of reward learning systems
\citep{Armstrong2020pitfalls,Everitt2019modeling,langlois2021rl,Evans2021,Farquhar2022}.
Typically, this analysis involves applying \emph{graphical criteria}, that indicate which properties can or
cannot occur in a given diagram, based on the graph structure alone.
Graphical criteria are useful because they
enable qualitative judgements 
even when the precise functional relationships between variables
are unknown or unspecified.
~\looseness=-1

For the single-decision case, complete criteria have been established for all four of the aforementioned concepts \citep{Everitt2021agent}.
However, many AI applications such as reinforcement learning involve an agent making multiple decisions.
For the multi-decision case, multiple criteria for VoI have been proposed
\citep{Nielsen1999,Shachter1998,nilsson2000evaluating},
but none proven complete.
\looseness=-1

\begin{figure}
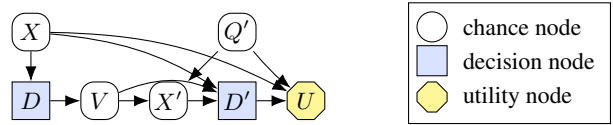

  \centering \setcompactsize
    
    \vspace*{3mm}
    \begin{influence-diagram}
    \setcompactsize
    \setcompactsize[node distance=0.1cm]
    \tikzset{
    every node/.append style = {node distance=5mm},
    }
        \node (Xs) [%
        ]{$X$};
        \node (Ds) [below = 0.4 of Xs, decision,%
        ]{$D$};
        \node (V) [right = 0.4 of Ds,%
        ]{$V$};
        \node (Xsp) [right = 0.4 of V,%
        ]{$X'$};
        \node (Dsp) [right = 0.4 of Xsp,decision,%
        ]{$D'$};
        \node (Qsp) [above = 0.4 of Dsp]{$Q'$};
        \node (U) [right=0.4 of Dsp,utility,%
        ]{$U$};
        \edge {Xs}{Ds};
        \path (Xs) edge[->, bend left=13] (Dsp);
        \path (Xs) edge[->, bend left=13] (U);
        \edge {Ds}{V};
        \edge {V}{Xsp};
        \path (V) edge[->, bend left=23] (Dsp);
        \edge {Xsp}{Dsp};
        \edge {Dsp}{U};
        \edge []{Qsp}{Xsp};
        \path (Qsp) edge[->, bend right=5] (U);
        
\cidlegend[right = of U, yshift=5mm, xshift=1mm]{
  \legendrow              {}          {chance node} \\
  \legendrow              {decision}  {decision node}\\
  \legendrow              {utility}   {utility node}}
    \end{influence-diagram}
    \caption{
    Does $X$ has positive value of information for $D$?
    }\label{fig:no-voi-graph}
  \end{figure}

This means that for some graphs, it is not known whether a node can have positive VoI.
For example, in \cref{fig:no-voi-graph}, it is not known whether it can be valuable for $D$
to observe $X$.
Specifically, the edge $X \to D$ does not meet the criterion of \emph{nonrequisiteness} used by \citet{nilsson2000evaluating}, 
so we cannot rule out that it contains valuable information.
However, the procedure that is used to prove completeness in the single-decision setting \citep{Everitt2021agent} does not establish positive VoI.

We prove that the graphical criterion of \citet{nilsson2000evaluating} is complete,  in that any environmental variable not guaranteed to have zero VoI by their criterion must have positive VoI in some compatible ID.
In the course of the proof, we develop several tools for reasoning about soluble IDs.
In summary, our main contributions are:
\begin{itemize}
    \item \textbf{ID homomorphisms}. These allow us to transform an ID into another with similar properties, that may be more easily analysed (\cref{sec:CID-homomorphisms-main}).
    \item \textbf{Trees of systems}. A system is a set of paths that make information valuable to a decision.
    A tree of systems describes how those paths traverse other decisions (\cref{sec:tree-main}).
    \looseness=-1
    \item \textbf{A complete VoI criterion}. We prove the criterion in \cref{sec:voi-completeness-main}. In \cref{sec:applications} we explain why this criterion may be useful, how it may be used in an AI safety application, and share an open source implementation.
\end{itemize}

\section{Setup}\label{sec:setup}

\ryan{Or ``
all of the results transfer to a regular influence diagram setting.
footnote: 
the only difference is that if edges do not match the direction of
causation, then a node may be deemed valuable to control, when controlling
it is not in-fact useful \citep[Appendix A]{everitt2021}
''}

Limited memory influence diagrams (also called LIMIDs) are graphical models containing 
decision and utility nodes, used to model decision-making problems \citep{Howard1966,nilsson2000evaluating}.

\begin{definition}[Limited memory influence diagram graph; \citealp{nilsson2000evaluating}]
    \label{def:cid}
    A \emph{(limited memory) ID graph}  is a directed acyclic graph
$\calG\!=\!(\sV,E)$ where the vertex set $\sV$ is partitioned into \emph{chance-} ($\sX$), \emph{decision-} ($\sD$), and \emph{utility
nodes} ($\sU$).
Utility nodes lack children.~\looseness=-1
\end{definition}

Since all of the influence diagram graphs in this paper have limited memory, 
we will consistently refer to them simply as \emph{influence diagram} (ID) graphs.
We denote the parents, descendants,\chris{And ancestors?} and family of a node $V \in \sV$
as $\Pa(V),\Desc(V)$, and $\Fa(V)=\Pa(V) \cup \{V\}$.
For $Y \in \sV$, \chris{Remove "For $Y \in \sV$,"?} we denote an edge by $V \to Y$, 
and a directed path 
by $V \pathto Y$.

To specify the precise statistical relationships, rather than just their structure, 
we will use a model that attaches probability distributions to the variables 
in an ID graph.

\begin{definition}%
    \label{def:cim}
    An \emph{influence diagram} (ID) is  a tuple $\calM = (\sG, \dom, P)$ where $\sG$ is an ID graph, $\dom(X)$ is a finite domain for each node $X$ in $\sG$ that is real-valued for utility nodes, and $P( X | \Pa(X))$ is a conditional probability distribution (CPD) for each chance and utility node $X$ in $\sG$.
    We will say that $\calM$ is \emph{compatible with} $\calG$, or simply that $\calM$ is an ID \emph{on} $\calG$.~\looseness=-1
\end{definition}

The decision-making task is to maximize the sum of expected utilities by selecting a CPD $\pi^D(D | \Pa(D))$,
called a \emph{decision rule}, 
for each decision $D \in \sD$. 
A \emph{policy} $\spi=\{\pi^D\}_{D \in \sD}$ consists of one decision rule for each decision.
Once the policy is specified, 
this induces joint probability distribution $P^\calM_\spi$ over all the variables.
We denote expectations by $\EE^\calM_\spi$ and omit the superscript 
when clear from context.
A policy $\spi$ is called \emph{optimal} if it
maximises $\EE_{\spi}[\totutilvar]$, where
$\totutilvar \coloneqq \sum_{U \in \sU}{U}$.
~\looseness=-1 %
Throughout this paper, we use subscripts for
policies, 
and superscripts for indexing.
A lowercase $v \in \dom(V)$ denotes an outcome of $V$.

Some past work has assumed ``no-forgetting'', meaning that every decision $d$ 
is allowed to depend on the value $v$ of any past decision $D'$ or 
its observations $\Pa(D')$, even when that variable $V \in \Fa(D')$ is not a parent of the current decision ($V \not \in \Pa(D)$) \citep{shachter1986evaluating}.
In contrast, we follow the more flexible convention of limited memory IDs \citep{nilsson2000evaluating}, 
by explicitly indicating whether a decision $d$ can depend on the value of an observation or decision $v$ 
by the presence (or absence) of an edge $V \to D$, just as we would do with any variable
that is not associated with a past decision.

Within the space of limited memory IDs, this paper focuses on \emph{soluble} IDs \citep{nilsson2000evaluating}, 
also known as IDs with ``sufficient recall'' \citep{Milch2008}.
The solubility assumption requires that 
it is always possible to choose an optimal decision rule without knowing what decision rules were followed by past decisions.
The formal definition uses $d$-separation.~\looseness=-1

\begin{definition}[d-separation; \citet{Verma1988soundness}]%
\label{def:d-separation}
    A path $p$ is \emph{blocked} by a set of nodes $\sZ$ if $p$ contains a collider $X \to W \gets Y$, such that neither $W$ nor any of its descendants are in $\sZ$, 
    or $p$ contains a chain $X \to W \to Y$ or fork $X \gets W \to Y$ where $W$ is in $\sZ$.
    If $p$ is not blocked, then it is \emph{active}.
    For disjoint sets $\sX$, $\sY$, $\sZ$, the set $\sZ$ is said to \emph{d-separate} $\sX$ from $\sY$,
    ${(\sX \dsep \sY \mid \sZ)}$ if $\sZ$ blocks every path
    from a node in $\sX$ to a node in $\sY$. Sets that are not d-separated are
    called \emph{d-connected}.~\looseness=-1
\end{definition}

\begin{definition}[Solubility; \citet{nilsson2000evaluating}]
  For an ID graph $\calG$ 
  let the \emph{mapping extension} $\calG'$ be 
  a modified version of $\calG$ where a chance node parent $\Pi^i$ is added to each decision $D^i$.
  Then $\calG$ is 
  \emph{soluble} if there exists an ordering $D^1, \dots, D^n$ over the decisions, such that in the mapping extension $\calG'$, for all $i$:~\looseness=-1 
  \[
    \sPi^{<i} \dsep \sU(D^i) \mid \Fa(D^i)
  \]
  where $\sPi^{<i}:=\{\Pi^j\mid j < i\}$ and
  $\sU(D^i) := \sU\cap\Desc(D^i)$.
\end{definition}

We will subsequently only consider ID graphs that are soluble.
Solubility is entailed by the popular more restrictive ``no forgetting'' assumption, where
the decision-maker remembers previous decisions and observations \citep{shachter1986evaluating,Shachter2016}:
in no forgetting, 
the family $\Fa(D^i)$ includes $\Fa(D^j)$ for $j<i$, so every policy node $\Pi^j$ is $d$-separated from $\sV \setminus \Fa(D^j) \supseteq \sU \cap \Desc^{D^j}$.
However, solubility is more general, for example \cref{fig:no-voi-graph} is soluble, even though past decisions are forgotten.

\section{Value of Information} \label{sec:voi-part-1}
The VoI of a variable
\ryan{Add ``is a widely studied property that [cite cite]''?}
indicates how much the attainable expected utility increases when a variable is observed compared to when it is not:

\begin{restatable}[Value of  Information; \citet{Howard1966}]{definition}{defvoi}
For an ID $\calM$ and $X \!\not \in \Desc_D$,
\ryan{I added a bit about it being a nondescendant.}
let $\calM_{X \!\not\! \to\! D}$ and $\calM_{X \!\to\! D}$ be
$\calM$ modified by respectively removing and adding the edge $X \!\to\! D$.
Then, the \emph{value of information} of $X$ for $D$ is:~\looseness=-1 
\[
\max_\spi \EE^{\calM_{X \to D}}_\spi[\totutilvar]
-
\max_\spi \EE^{\calM_{X \not \to D}}_\spi[\totutilvar]
.\]
\end{restatable}

This is closely related to the concept of \emph{materiality}; an observation $X \in \Pa(D)$ is called 
material if its VoI is positive.

The graphical criterion for VoI that
we will use iteratively removes information links that cannot contain useful information,
based on a condition called \emph{nonrequisiteness}.
If $X\dsep \sU(D^i)\mid \Fa(D^i)\setminus\{X\}$, then both $X$ and the information link $X\to D^i$ 
are called \emph{nonrequisite}, otherwise, they are \emph{requisite}.
Intuitively, nonrequisite links contain no information about 
influencable utility nodes, so the attainable expected utility is not 
decreased by their removal.
Removing one nonrequisite observation link can make a previously requisite information link nonrequisite, 
so the criterion involves iterative removal of nonrequisite links.
The criterion was first proposed by \citet{nilsson2000evaluating}, who also proved that it is sound. 
Formally, it is captured by what we calll a $d$-reduction:

\begin{definition}[$d$-reduction]
\label{def:d-reduction}
  The ID graph $\calG'$ is a \emph{$d$-reduction} of $\calG$ if $\calG'$ can be obtained from
  $\calG$ via a sequence $\calG=\calG^1,...,\calG^k=\calG'$
  where each $\calG^i,i>1$ differs from its predecessor $\calG^{i-1}$
  by the removal of one nonrequisite information link.
  A $d$-reduction is called \emph{minimal} if it lacks any nonrequisite information links.~\looseness=-1 
\end{definition}

For any ID graph $\calG$, there is only one minimal $d$-reduction \citep{nilsson2000evaluating},
i.e.\ the minimal $d$ reduction is independent of the order in which edges are removed.
We can therefore denote \emph{the \minimaldred} of $\calG$
as $\calG^*$.
Thus, \citet[Theorem 3]{nilsson2000evaluating} states that \emph{if} an ID graph $\calG$ contains $X \to D$
but %
$\calG^*$ does not, then $X$ has zero VoI in
every ID compatible with $\calG$.
Our completeness result replaces this with an \emph{if and only if} statement.
~\looseness=-1

\begin{restatable}[VoI Criterion]{theorem}{thmvoi} \label{thm:voi}
Let $\calG$ be a soluble ID graph containing an edge $X \to D$ from
chance node $X \in \sX$ to decision $D \in \sD$.
There exists an ID $\calM$ compatible with $\calG$ such that $X$ has strictly positive VoI for $D$ if and only if 
the minimal $d$-reduction contains $X \to D$.
\end{restatable}

The VoI criterion is posed in terms of a graph $\calG$ that contains $X \to D$.
To analyse a graph that does not, %
one can simply add the edge $X \to D$ then apply the 
same criterion
as long as the new ID graph is soluble \citep{Shachter2016}.
\ryan{Is this the correct shachter cite?}

The proof will be given in \cref{sec:voi-completeness-main}, with details in \cref{sec:preliminaries-systems-and-trees,sec:model-definition}.
We note that this excludes the case of remembering a past decision $X \in \sD$, 
because 
Nilsson's criterion
is incomplete for this case.
For example, the simple ID graph with the edges $D\to D'\to U$ and $D\to U$, $D$ satisfies the graphical criterion of being requisite for $D'$, but $D'$ has zero VoI because it is possible for the decision $D$ to be deterministically assigned some optimal value. This means that there is no need for $D'$ to observe $D$.~\looseness=-1

\section{ID Homomorphisms} \label{sec:CID-homomorphisms-main}

To make the analysis easier, we will often want to transform an original ID graph into a more structured one.
Before describing the structure we will be aiming for, we consider the general question of when a modified ID graph retains important properties of the original.
To this end, we will define  the concept of an \emph{ID homomorphism}, 
which we then use to define a class of property-preserving ID transformations. (Proofs are supplied in \cref{sec:CID-homomorphisms}.)%

\newcommand{\preservesnodetypes}{Preserves node types}
\newcommand{\preserveslinks}{Preserves links}
\newcommand{\coversallinfolinks}{Covers all infolinks}
\newcommand{\combinesonlylinkeddecisions}{Combines only linked decisions}

\ryan{Should be able to condense the bullets from 8 lines to ~5}
\begin{restatable}[ID homomorphism]{definition}{defcidhomomorphism}
\label{def:CID-homomorphism}
For ID graphs $\calG\!\!=\!(\sV\!,E)$ and $\calG'\!\!=\!(\sV'\!,E')$, a map $h\colon\!\sV' \!\!\to\! \sV\!$ is an \emph{ID homomorphism} from $\sG'$ to $\sG$
iff:
\begin{enumerate}[label=(\alph*)]
    \item 
    (Preserves node types)
    $h$ maps each chance-, decision-, or utility-node to a node of
    the same type; 
    
    \item 
    (Preserves links)
    For every 
    $A \to B$ in $\calG'$
    either
    $h(A) \to h(B)$ is in $\calG$, or $h(A)=h(B)$; 

    \item 
    (Covers all information links)
    If $h(N) \to h(D)$ is in $\calG$ for $D\in \sD$, then $N \to D$ is in $\calG'$; and

    \item 
    (Combines only linked decisions)
    If $h(D_1)\!=\!h(D_2)$ for decisions $D^1\neq D^2$ in $ \calG'$
    then $\calG'$ contains $D^1 \!\!\to\! D^2$ or $D^2 \!\to\! D^1$.
\end{enumerate}
\end{restatable}

\begin{figure}
    \centering
    \begin{influence-diagram}
    \setcompactsize
    
    \node (Y1) {$Y$};
    \node (D1) [decision, right = of Y1] {$D$};
    \node (U1) [utility, below = of D1] {$U$};
    \edge{Y1}{D1};
    \edge{D1}{U1};
    \node (G1) [fit={(Y1) (D1) (U1)}, inner sep=2mm] {};
    \node [rectangle, draw=none, below = 0.1 of G1] {$\calG$\\ original};
    
    \node (D2) [decision, right = 1 of D1] {$D$};
    \node (U2) [utility, below = of D2] {$U$};
    \edge{D2}{U2};
    \node (G2) [fit={(D2) (U2)}, inner sep=2mm]
    {};
    \node [rectangle, draw=none, below = 0.1 of G2] {$\calG'$\\ remove $Y$};
    
    \node (D3) [decision, right = 1 of D2] {$D$};
    \node (D3p) [decision, right = of D3] {$D'$};
    \node (U3) [utility, below = of D3] {$U$};
    \edge{D3,D3p}{U3};
    \edge{D3}{D3p}
    \node (G3) [fit={(D3) (D3p) (U3)}, inner sep=2mm]
    {};
    \node [rectangle, draw=none, below = 0.1 of G3] {$\calG''$\\ duplicate $D$};
    
    \node (D4) [decision, right = 1 of D3p] {$D$};
    \node (D4p) [decision, right = of D4] {$D'$};
    \node (U4) [utility, below = of D4] {$U$};
    \edge{D4}{U4};
    \edge{D4}{D4p}
    \node (G4) [fit={(D4) (D4p) (U4)}, inner sep=2mm]
    {};
    \node [rectangle, draw=none, below = 0.1 of G4] {$\calG'''$\\ remove an edge};
    
    \path (D2) edge[->, bend right, blue, thick] (D1);
    \path (U2) edge[->, bend right, blue, thick] (U1);
    
    \path (D3) edge[->, bend right, green, thick] (D2);
    \path (D3p) edge[->, bend left, green, thick] (D2);
    \path (U3) edge[->, bend left, green, thick] (U2);
    
    \path (D4) edge[->, bend right, orange, thick] (D3);
    \path (D4p) edge[->, bend left, orange, thick] (D3p);
    \path (U4) edge[->, bend left, orange, thick] (U3);
    
    \end{influence-diagram}
    \caption{A sequence of homorphic transformations showing how $\calG$ can be homorphically transformed into $\calG'''$ by composition of \cref{le:21may19.2-CID-hom-from-node-copying-and-deleting,le:21may19.2-CID-hom-from-edge-pruning}.
    In the first step from $\calG$ to $\calG'$, $Y$ is removed; in the step from $\calG'$ to $\calG''$ a decision is duplicated; and in the final step from $\calG''$ to $\calG'''$, a link is removed.
    Since 
    the mapping at each step (blue, green, and orange respectively)
    meets the definition of an ID homomorphism, $\calG'''$ must be an ID homorphism of $\calG$ (\cref{le:20dec7.1-composition-of-CID-splits}).
    }
    \label{fig:homorphism}
\end{figure}

An ID homomorphism is analogous to the notion of graph homomorphism from graph theory, which essentially requires that edges are preserved along the map. 
An ID homomorphism additionally
requires that decisions in the two graphs have equivalent parents (c), and that split decisions are connected (d). This requirement maintains a direct correspondence between policies on the two graphs, so that, as we will see, ID homomorphisms preserve VoI.
Examples of ID homorphisms are given in \cref{fig:homorphism}.
~\looseness=-1

\ryan{Cite graph homomorphisms?}
\chris{we could add a citation to graph theory, Diestel, 2017. Though it's also just a ``well known concept" and has a wikipedia page.}
\ryan{This is where we should have remarks about the intuition of this definition. But I don't understand 
what is being said about adding edges from each node to itself.}

The following three lemmas establish properties that are preserved under
ID homorphisms. %
\begin{restatable}[Preserves Solubility]{lemma}{lecidhomsufficientrecall} \label{20nov25.1-CID-homomorphism-preserves-sufficient-recall-SR}
Let $\sG=(\sV,E)$ and $\sG'=(\sV',E')$ be ID graphs. If $\sG$ is soluble, and there exists a homomorphism $h\colon\sV'\to\sV$, then $\sG'$ is also soluble.
\end{restatable}

\ryan{Probably we should uniformise to either $\calG'/\calM'$ or $\bar \calG \bar \calM$ throughout this section}

Given a homomorphism $h$ from $\sG'$ to $\sG$, we can define a notion of 
equivalence between IDs (and policies) on each graph.
Roughly, two IDs are equivalent if the domain of every node 
is a cartesian product of the domains of the nodes in its pre-image (or the 
sum, in the case of a utility node).
Formally:

\begin{definition}[Equivalence]
$\calM_\pi$ on $\calG$ and $\calM'_{\pi'}$ on $\calG_{\pi'}$ are \emph{equivalent} if
each non-utility node $N$ in $\sG$ has $\dom(N):=\bigtimes_{N^i \in h^{-1}(N)} \dom(N^i)$,
and $P^\calM_\pi(N\!=\!(n^1,...,n^k))=P^{\calM'}_{\pi'}(N^1\!=\!n^1,...,N^k\!=\!n^k)$, 
and each utility node has $P^\calM_\pi(U\!=\!u)=P^{\calM'}_{\pi'}(\sum_{U^i\in h^{-1}(U)} U^i\! =\! u)$.
\end{definition}

\begin{restatable}[Equivalence]{lemma}{lecidhomequivalence} \label{le:cidhom1-equivalence}
If there is an ID homomorphism $h$ from $\sG '$ to $\sG$,
then for any policy $\pi'$ in any ID $\calM'$ on $\sG '$ 
there is a policy $\pi$ in a ID $\calM$ on $\sG$ 
such that $\calM_\pi$ and $\calM'_{\pi'}$ are equivalent. 
\tom{what does it mean for two probability distributions to be equivalent?}
\end{restatable}
\ryan{Changed from $\tilde M$ to $M'$ here. May need to make corresponding change to proof in appendix}

In this case, we will call $\calM$ and $\pi$ the \emph{ID and policy transported along the homomorphism $h$}. In the appendix, we show that this correspondence between policies on $\calM'$ and $\calM$ is a bijection.
\ryan{The one-sentence explanation above isn't very explanatory.}
\chris{I just removed it}
Intuitively, if there is an ID homomorphism $\sG'\to \sG$, this means we have a particular way to \textit{fit an ID on $\sG'$ into $\sG$}, while preserving the information that the decisions can access.
The basis of this proof is that properties (c,d) of ID homomorphisms (\cref{def:CID-homomorphism}) require decisions to have precisely the same information in $\cal M$ as in $\cal M'$.~\looseness=-1

\ryan{``interpreted as'' feels a bit too informal to me.}

For our proof of \cref{thm:voi}, we will require that VoI is preserved under homomorphism.

\begin{restatable}[Preserves VoI]{lemma} {lecidhompreservesmateriality} \label{th:CID-homomorphism-preserves-Materiality}
    Let $h\colon\!\sG'\!\!\to\! \sG$ be an ID homomorphism.
    If $X'$ has positive VoI for $D'$ in an ID $\calM'$ on $\sG'$, then $X\!=\!h(X)$ has positive VoI for $D\!=\!h(D')$ in the transported ID $\calM\!=\!h(\calM')$.~\looseness=-1
\end{restatable}
\ryan{I thought materiality is defined with respect to just a model, so I've changed this statement. I think the proof should remain similar, and use transported model + policy?}

The proof builds heavily on there being a precise correspondence between policies on $\cal M$ and on $\cal M'$. Since these two IDs are equivalent (\cref{le:cidhom1-equivalence}), if obtaining certain information in $\cal M'$ has value, so does obtaining that information in $\cal M$. The formal details are left to \cref{sec:CID-homomorphisms}.~\looseness=-1

We next present two transformation rules\chris{rudimentary calculus is kind of a weird phrasing} with which to modify any ID graph, which are illustrated in \cref{fig:homorphism}.
The first transformation obtains a new graph $\sG'$ by deleting or duplicating nodes, while preserving all links. Under this transformation, the function that maps a node in $\sG'$ to its `originating node' in $\sG$ is an ID homomorphism:
~\looseness=-1

\newcommand{\sCopies}{\mathrm{Copies}}
\begin{restatable}[Deletion \& Link-Preserving Copying]{lemma}{lecopyingcidhom} \label{le:21may19.2-CID-hom-from-node-copying-and-deleting}
Let $\sG\!\!=\!\!(\sV, E)$ be an ID graph and $\sG'\!=\!(\bigcup_{N\in\sV}\sCopies(N),E')$ an ID graph where $\sCopies$ maps nodes in $\sG$ to disjoint sets in $\sG'$, and where $E'$ is a minimal set of edges such that 
for any edge $A \to B$ in $E$ and $A^i\in \sCopies(A)$ and $B^i\in \sCopies(B)$ there is an edge $A^i\to B^i$,
and if  $A^i,A^j\in \sCopies(A)$ are non-utility nodes then 
either
$A^i\to A^j$ or $A^i\gets A^j$. Then the function $h$ that maps each $V\in \sCopies(N)$ to $N$ is an ID homomorphism.~\looseness=-1

\end{restatable}
\ryan{I've simplified/shortened this a bit further. Feel free to revert any changes is preferred.}

Edges that are not information links can also be removed, while having a homomorphism back to the original:

\begin{restatable}[Link Pruning]{lemma}{lepruningcidhom} \label{le:21may19.2-CID-hom-from-edge-pruning}
Let $\sG=(\sV,E)$ and $\sG'=(\sV,E')$ be ID graphs, where $E'\subseteq E$ and where for each decision node $D$ in $\sV$, every incoming edge $N\to D$ in $E$ is in $E'$. Then the identity function $h(N)=N$ on $\sV$ is a homomorphism from $\sG'$ to $\sG$.
\end{restatable}

Finally, we can chain together a sequence of such graph transformation steps, and still maintain a homomorphism to the original. The justification for this is that a composition of ID homomorphisms is again an ID homomorphism:~\looseness=-1

\begin{restatable}[Composition]{lemma}{lecidhomcomposition} \label{le:20dec7.1-composition-of-CID-splits}
If $h\colon\sG' \to \sG$ and $h'\colon\sG'' \to \sG'$ are ID homomorphisms then the composition $h \circ h'\colon\sG'' \to \sG$ is an ID homomorphism.
\end{restatable}
\section{Completeness of the VoI Criterion} \label{sec:voi-completeness-main} %
We will now prove that the 
\emph{value of information} (VoI)
criterion 
of \citet{nilsson2000evaluating}
is complete for chance nodes
(details are deferred to \cref{sec:preliminaries-systems-and-trees,sec:model-definition}).

\subsection{Parameterising one system} \label{sec:tree} %
To prove 
that the criterion from \cref{thm:voi} is complete
we must show that for any graph where $X\to D$ is in the minimal d-reduction, 
$X$ has positive VoI for $D$.
For example, consider the graph in \cref{fig:trivial-completeness}, 
which is its own d-reduction, and contains $X \to D$.
In this graph, we can choose for $X$ to be Bernoulli distributed, 
for $D$ to have the boolean domain $\{0,1\}$, and for 
$U$ to be equal to $1$ if and only if $X$ and $D$ match.
Clearly, the policy $d=x$ will obtain $\mathbb{E}[U]=1$.
In contrast, 
if $X$ were not observed (no link $X\to D)$, then no policy could achieve expected utility more than $0.5$;
so the VoI of $X$ in this ID is $0.5$.~\looseness=-1

\begin{figure}[H]
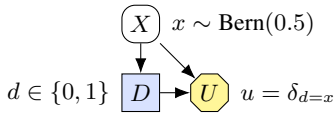

\centering
    \begin{influence-diagram}
    \setcompactsize
    \setcompactsize[node distance=0.1cm]
    \tikzset{
    every node/.append style = {node distance=5mm},
    }
        \node (Xs) [label={[label distance=1mm]0:\footnotesize$x \sim \text{Bern}(0.5)$}]{$X$};
        \node (Ds) [below = 0.4 of Xs, decision,label={[label distance=1mm]180:{\footnotesize $d \in \{0,1\}$}}]{$D$};
        \node (U) [right=0.4 of Ds,utility,label={[label distance=1mm]0:{\footnotesize $u=\delta_{d=x}$}}]{$U$};
        \edge {Xs}{Ds};
        \edge {Ds}{U};
        \edge {Xs}{U};

    \end{influence-diagram}
        \caption{
        The observation $X$ has positive VoI for $D$.
  }
  \label{fig:trivial-completeness}
\end{figure}

A general procedure for parameterising any single-decision ID graph meeting the \cref{thm:voi} criterion to exhibit positive VoI has been established by \citet{Everitt2021agent} and \citet{lee2020characterizing}.
This procedure consists of two steps:
first, establish the existence of some paths, 
then choose CPDs for the nodes on those paths.
We call the paths found in the first step a \textit{system}, which will be a building block for our analysis of IDs with multiple decisions.
A fully-general illustration of a system is shown in \cref{fig:regions}.~\looseness=-1
\ryan{Maybe move some of these definitions down to wherever they're used.}

\begin{restatable}[System]{definition}{defsystem} \label{def:system-minimal}
A \emph{system} $s$ in an ID graph $\calG$ %
is a tuple $(\scontrol^s, \sinfo^s, \sobspaths^s)$ where: 
\begin{itemize}
    \item 
    The \emph{control path}, $\scontrol^s$, is a directed path $D^s \pathto U^s$ where $D^s \in \bm{D}$ and $U^s\in \bm{U}$,
    \item The \emph{info path}, $\emph{$\sinfo^s$}$, is a path $\Pa(D^s)\ni X^s\upathto U^s$, active given $\Fa(D^s) \setminus \{X^s\}$,
    \item $\sobspaths^s$ maps each collider $C^i$ in $\sinfo^s$ to
    an \emph{obs path}, a \emph{minimal-length} directed path $C^i \!\!\dashrightarrow\! D^s$.~\looseness=-1
\end{itemize}

\ryan{Tweaked the formatting of this definition, old version below}

We denote the \emph{information link of} $s$, $X^s \!\to\! D^s$, by
$\sinfolink^s$ and the union of nodes in \emph{all} paths of $s$ 
by $\sV^s$.
\end{restatable}

\ryan{TODO: set nodes to be circular, except for labels, which are rectangular}
\begin{figure}[H]
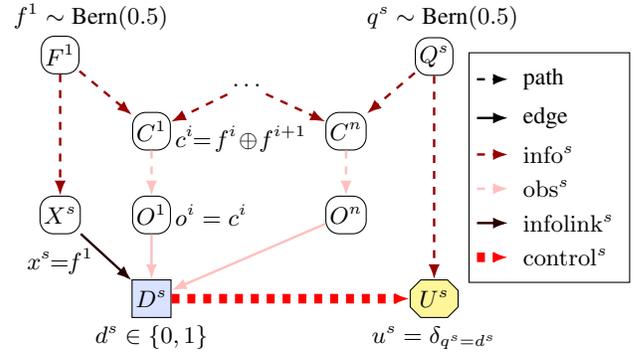

  \centering
  \begin{influence-diagram}
    \setcompactsize[node distance=0.1cm]
    \tikzset{
    path/.append style = {black,dashed, line width=0.9pt},
    link/.append style = {black,line width=0.9pt},
    infopath/.append style = {dashed,black!40!red,line width=0.9pt},
    obspath/.append style = {dashed,pink,line width=0.9pt},
    controlpath/.append style = {dashed,red,line width=3.6pt},
    infolink/.append style = {black!85!red,line width=0.9pt},
    every node/.append style = {node distance=7mm},
    }
        \node (b0) [label={[label distance=0mm,xshift=-7mm]85:\footnotesize$f^1 \sim \text{Bern}(0.5)$}] {$F^1$};
        \node (C1) [below right = 0.5 and 0.7 of b0,label={[label distance=0mm,yshift=-1mm]0:\footnotesize$c^i\!\!=\!f^i \!\xor\! f^{i+1}$}] {$C^{1}$};
        \node (dots1) [above right = 0.1 and 0.75  of C1, draw=none] {$\hdots$};
        \node (Cn) [below right = 0.1 and 0.79 of dots1,%
        ] {$C^{n}$};
        \node (Qs) [above right = 0.48 and 0.64 of Cn,label={[label distance=0mm,xshift=12mm]95:\footnotesize$q^s \sim \text{Bern}(0.5)$}] {$Q^s$};
        \node (O1) [below = 0.6 of C1,label={[label distance=0mm]0:\footnotesize$o^i=c^i$}]{$O^1$};
        \node (On) [below = 0.6 of Cn,%
        ]{$O^n$};
        \node (Xs) at (b0|-O1) [label={[label distance=1mm,xshift=0mm]-90:\footnotesize$x^s\!\!=\!\!f^1$}]{$X^s$};
        \node (Ds) [below = 0.6 of O1, decision,
        label={[label distance=0mm]-90:\footnotesize$d^s\in \{0,1\}$}
        ]{$D^s$};
        \node (Us) at (Qs|-Ds) [utility,label={[label distance=0mm]-90:\footnotesize$u^s=\delta_{q^s=d^s}$}]{$U^s$};
        
        \edge [infopath]{b0}{C1};
        \edge [infopath]{dots1}{C1};
        \edge [infopath]{dots1}{Cn};
        \edge [infopath]{Qs}{Cn};
        
        \edge [infopath]{b0}{Xs};
        \edge [obspath]{C1}{O1};
        \edge [obspath]{Cn}{On};
        
        \edge [infolink]{Xs}{Ds};
        \edge [obspath,solid]{O1}{Ds};
        \edge [obspath,solid]{On}{Ds};
        
        \edge [controlpath]{Ds}{Us};
        \edge [infopath]{Qs}{Us};
        
        \cidlegend[below right = 0.2 and 0.2 of Qs.north east, anchor=north west]{
        \legendrow[path] {draw=none} {path} \\ 
        \legendrow[link] {draw=none} {edge} \\ 
        \legendrow[infopath] {draw=none} {$\sinfo^s$} \\ 
        \legendrow[obspath] {draw=none} {$\sobspaths^s$} \\
        \legendrow[infolink] {draw=none} {$\sinfolink^s$} \\ 
        \legendrow[controlpath] {draw=none} {$\scontrol^s$} \\
        }
        \edge[path]{path.west}{path.east}
        \edge[link]{link.west}{link.east}
        \edge[obspath]{obspath.west}{obspath.east}
        \edge[infopath]{infopath.west}{infopath.east}
        \edge[controlpath]{controlpath.west}{controlpath.east}
        \edge[infolink]{infolink.west}{infolink.east}
  \end{influence-diagram}
\caption{
A system, annotated with a parameterization that has positive VoI in the single-decision case.
Dashed arrows can zero or more nodes.
}
  \label{fig:regions}
\end{figure}
\ryan{Probably should try to make the text larger in this figure if spare space}

The existence of these paths follow from the graphical criterion of \cref{thm:voi}.
In particular, since $X \!\to\! D$ is in the minimal d-reduction of $\calG$, 
there must exist a path
from $X$ to some utility node $U \in \sU \cap \Desc^{D^s}\!$, active given $\Fa(D^s) \!\setminus\! \{X^s\}$ (the ``info path" in \cref{def:system-minimal}).

The second step is to choose CPDs for the nodes $\sV^s$ in the system $s$, as also illustrated in \cref{fig:regions}.
The idea is to require
the decision $D^s$ to match the value of $Q^s$, by letting
the utility $U^s$ equal $1$ if and only if its parents along the control 
and information paths are equal.
If $X^s$ is observed, the decision $D^s=X^s \oplus O^1 ... \oplus O^n=Q^s$
yields $\mathbb{E}[U^s]=1$, where $\oplus$ denotes \emph{exclusive or} (XOR).
Otherwise, the observations $O^1,...,O^n$ are insufficient to decrypt $Q^s$, 
giving $\mathbb{E}[U^s]<1$.
So $X^s$ has positive VoI.
The intuitive idea is that $U^s$ tests whether $D^s$ knows $Q^s$, 
based on the value $d^s$ transmitted along $\scontrol^s$.~\looseness=-1
\chris{Above is a bit hard to follow I think.}

\subsection{Parameterising two systems} 
\begin{figure}
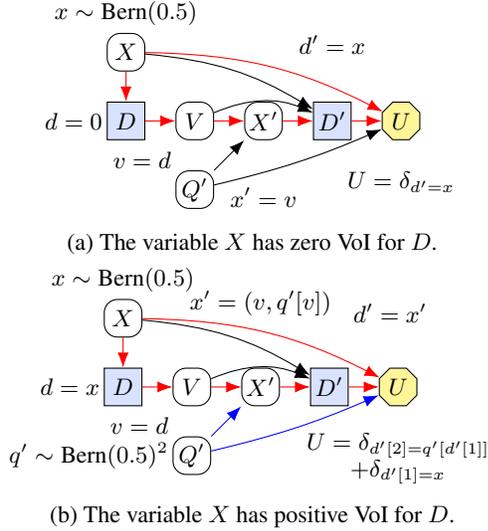

\vspace{-5mm}
\centering
\begin{subfigure}[t]{0.47\textwidth}
  \centering \setcompactsize
  \begin{influence-diagram}
    \end{influence-diagram}
    
    \vspace*{3mm}
    \begin{influence-diagram}
    \setcompactsize
    \setcompactsize[node distance=0.1cm]
    \tikzset{
    every node/.append style = {node distance=5mm},
    }
        \node (Xs) [ label={[label distance=0mm]90:\footnotesize$x \sim \text{Bern}(0.5)$}
        ]{$X$};
        \node (Ds) [below = 0.4 of Xs, decision,label=left:\footnotesize{$d=0$}
        ]{$D$};
        \node (V) [right = 0.4 of Ds,label=below left:\footnotesize{$v=d$}
        ]{$V$};
        \node (Qsp) [below = 0.4 of V]{$Q'$};
        \node (Xsp) [right = 0.4 of V,label={[label distance=5mm]270:\footnotesize{$x'=v$}}
        ]{$X'$};
        \node (Dsp) [right = 0.4 of Xsp,decision,label={[label distance=5mm]90:\footnotesize{$d'=x$}}
        ]{$D'$};
        \node (U) [right=0.4 of Dsp,utility,label={[label distance=3mm]270:{\footnotesize $U=\delta_{d'=x}$}}
        ]{$U$};
        label={below:[align=left]\footnotesize{}}
        \edge [red]{Xs}{Ds};
        \path (Xs) edge[->, bend left=13] (Dsp);
        \path (Xs) edge[->, red, bend left=13] (U);
        \edge [red]{Ds}{V};
        \edge [red]{V}{Xsp};
        \path (V) edge[->, bend left=23] (Dsp);
        \edge [red]{Xsp}{Dsp};
        \edge [red]{Dsp}{U};
        \edge []{Qsp}{Xsp};
        \path (Qsp) edge[->, bend right=5] (U);

    \end{influence-diagram}
    \caption{
    The variable $X$ has zero VoI for $D$.}\label{fig:multi-dec-param-1}
\end{subfigure}

\begin{subfigure}[t]{.47\textwidth}
\centering
    \begin{influence-diagram}
    \setcompactsize
    \setcompactsize[node distance=0.1cm]
    \tikzset{
    every node/.append style = {node distance=5mm},
    }
        \node (Xs) [label={[label distance=0mm]90:\footnotesize$x \sim \text{Bern}(0.5)$}]{$X$};
        \node (Ds) [below = 0.4 of Xs, decision,label=left:\footnotesize{$d=x$}]{$D$};
        \node (V) [right = 0.4 of Ds,label=below left:\footnotesize{$v=d$}]{$V$};
        \node (Qsp) [below = 0.4 of V,label=left:\footnotesize{$q'\sim \text{Bern}(0.5)^2$}]{$Q'$};
        \node (Xsp) [right = 0.4 of V,label={[label distance=6mm]90:\footnotesize{$x'=(v,q'[v])$}}]{$X'$};
        \node (Dsp) [right = 0.4 of Xsp,decision,label={[label distance=5mm]70:\footnotesize{$d'=x'$}}]{$D'$};
        \node (U) [right=0.4 of Dsp,utility,label={[label distance=3mm]270:{\footnotesize $U=\delta_{d'[2]=q'[d'[1]]}$\\$+\delta_{d'[1]=x}$}}]{$U$};
        \edge [red]{Xs}{Ds};
        \path (Xs) edge[->, bend left=13] (Dsp);
        \path (Xs) edge[->, red, bend left=13] (U);
        \edge [red]{Ds}{V};
        \edge [red]{V}{Xsp};
        \path (V) edge[->, bend left=23] (Dsp);
        \edge [red]{Xsp}{Dsp};
        \edge [red]{Dsp}{U};
        \edge [blue]{Qsp}{Xsp};
        \path (Qsp) edge[->, blue, bend right=5] (U);
    \end{influence-diagram}
        \caption{
        The variable $X$ has positive VoI for $D$.
  }
  \label{fig:multi-dec-param-2}
\end{subfigure}
\caption{In (a), a parameterisation of nodes in a single (red) system fails to exhibit that $X$ has positive VoI for $D$, 
whereas in (b), positive VoI is exhibited by parameterising two (red and blue) systems.}

\end{figure}

When we have two decisions, however, it becomes insufficient to parameterise just one system.
For example, suppose that we try to apply 
the same scheme as in the previous subsection
to the graph of \cref{fig:multi-dec-param-1}.
Then, we would generate a random bit at $X$ and stipulate that the utility is $U=1$
if the parents $X$ and $D'$ on the red paths are equal.
One might hope that this would give $D$ an incentive to observe $X$, so that $d=x$ is copied through $D'$ to obtain $\mathbb{E}[U]=1$.
And that is indeed one way to obtain optimal expected utility.
However, the presence of a second decision $D'$ means that maximal utility of $U=1$ may also be obtained using the policy $d=0,d'=x$, 
which does not require $X$ to be observed by $D$.~\looseness=-1

To achieve positive VoI, it is necessary to parameterise two systems as shown in \cref{fig:multi-dec-param-2}.
We first parameterise the second (blue) system 
to ensure that $x'$ is transmitted to $U$, 
and then parameterise the initial (red) system. 

To check that $X$ has positive VoI for $D$, we now solve the combined model.
Due to the solubility assumption, we know that the optimal 
decision rule at $D'$ does not depend on the decision rule taken at $D$.
So let us consider $D'$ first.
$D'$ chooses a pair $(i, j)$ where $i$ is interpreted as an index of the bits generated at $Q'$, and $j$ is interpreted as a claim about the $i\textsuperscript{th}$ bit of $Q'$.
The first term of the utility $U$ is equal to $1$ if and only if the ``claim" made by $D'$ is correct, i.e.\ if the $i\textsuperscript{th}$ bit generated by $Q'$ really is $j$.
$X'$ contains (only) the $v\textsuperscript{th}$ digit of $Q'$.
Hence $D'$ can only ensure its ``claim" is correct if it chooses $d'=x'=(v,q'[v])$, 
where $q'[v]$ denotes the $v\textsuperscript{th}$ bit of $q'$.
Having figured out the optimal policy for $D'$, we next turn our attention to $D$.
Intuitively, the task of $D$ is to match $X$, as in \cref{fig:trivial-completeness}.
The parameterization encodes this task, by letting $D$ determine $V$, which in turn influences which bit of $Q'$ is revealed to $D'$.
This allows $U$ to check the output of $D$ via the index outputted by $D'$, and thereby check whether $D$ matched $X$.
This means the second term of $U$ is 1 if and only if $D=X$
so $d=x$ the optimal policy for $D$, with expected utility $\mathbb{E}[U]=2$.\looseness=-1

In contrast, if $X$ were unobserved by $D$, then it would no-longer be possible to achieve a perfect score on both terms of $U$, so $\mathbb{E}[U]<2$.
This shows that $X$ has positive VoI for $D$.\looseness=-1

\subsection{A tree of systems} \label{sec:tree-main}
\todo{add explanation here}

In order to generalise this approach to arbitrary number of decisions, 
we 
need a structure that specifies a system for each decision, and indicates 
what downstream decisions that system may depend on.
These relationships may be represented by a tree.~\looseness=-1

\begin{restatable}[Tree of systems]{definition}{deftree} \label{def:tree-of-systems}
A \emph{tree of systems} on an ID graph $\calG$ is a tuple $T=(\calS,\spred)$ where:
\begin{itemize}
    \item $\calS = (s^0,...,s^k)$ is a list of systems (which may include duplicates).
    
    \item $\spred$ maps each $s^i$ to a pair $(s^j,p)$, 
    where $s^j\in (\calS\setminus\{s^i\})$ is a system, %
    $p$ is one of the paths of $s^j$ (info, control, or obs), and
    $\sinfolink^{s^i}$ is in the path $p$,
    except there is a unique ``root system'' $\ssRoot$ that is mapped to $(\ssRoot,``\snone")$.
\end{itemize}
Moreover, a \emph{full tree of systems} is one where for each information link $X' \to D'$ in each path $p$ in each system $s$, there is precisely one system $s'$ whose information link equals $X' \to D'$ and with $\pred{s'}=(s,p)$.
\end{restatable}

The idea of a tree of systems is that 
if a decision $D^{s'}$ lies on a path in the system $s$ 
of some decision $D^s$, then $s$ is a predecessor of $s'$.
We will use this tree to parameterise the ID graph,
and then we will also use it to supply an ordering over the decisions 
(from leaf to root) in which the model can be solved by backward induction.
\ryan{Added some explanation here.}

\begin{figure*}
  \centering
  \begin{subfigure}[t]{0.22\textwidth}
    \centering
\begin{influence-diagram}
\setcompactsize[node distance=0.5cm]
\node (X) [] {$X^s$};
\node (D) [below = of X,decision] {$D^s$};
\node (Y) [below = of D] {$Y$};
\node (Xp) [right = 4mm of Y] {$X'$};
\node (Dp) [right = 4mm of Xp,decision] {$D'$};
\node (U) [right = of Dp,utility] {$U$};
\node (Qp) at (U|-X) [] {$Q'$};
\node (h1) [draw=none,left = 1.5mm of Y, inner sep=0mm, minimum size=0mm] {};
\node (h2) [draw=none,below = 1.5mm of Y, inner sep=0mm, minimum size=0mm] {};
\path (X) edge[out=-120,in=95,red] (h1)
(h1) edge[out =-85, in=175,red] (h2)
(h2) edge[->, out =-5, in=-155,red] (U)
;

\edge {X}{Dp}
\edge {D}{Dp}
\path (Qp) edge[->, bend right=12,blue] (Y);

\edge[red] {X}{D};
\edge[red] {D}{Y};
\draw[->, red] ([yshift=-0.5mm]Xp.east) -- ([yshift=-0.5mm]Dp.west);
\draw[->, red] ([yshift=-0.5mm]Y.east) -- ([yshift=-0.5mm]Xp.west);
\draw[->, red] ([yshift=-0.5mm]Dp.east) -- ([yshift=-0.5mm]U.west);

\draw[->, blue] ([yshift=0.5mm]Xp.east) -- ([yshift=0.5mm]Dp.west);
\draw[->, blue] ([yshift=0.5mm]Y.east) -- ([yshift=0.5mm]Xp.west);
\draw[->, blue] ([yshift=0.5mm]Dp.east) -- ([yshift=0.5mm]U.west);
\edge[blue] {Qp}{U};

\end{influence-diagram}
  \caption{$Y$ and $U$ occur in both
  $s$ (red) and in $s'$ (blue)} \label{fig:transform-1}
\end{subfigure}\hspace{5mm}%
\begin{subfigure}[t]{0.22\textwidth}
  \centering
\begin{influence-diagram}
\setcompactsize[node distance=0.5cm]
\node (X) [] {$X^s$};
\node (D) [below = of X,decision] {$D^s$};
\node (Y) [below = of D] {$Y$};
\node (Xp) [right = 4mm of Y] {$X'$};
\node (Dp) [right = 4mm of Xp,decision] {$D'$};
\node (U) [right = of Dp,utility] {$U$};
\node (Yp) at (Dp|-X) {$Y'$};
\node (Up) at (U|-D) [utility] {$U'$};
\node (Qp) at (U|-X) [] {$Q'$};
\node (h1) [draw=none,left = 1.5mm of Y, inner sep=0mm, minimum size=0mm] {};
\node (h2) [draw=none,below = 1.5mm of Y, inner sep=0mm, minimum size=0mm] {};
\path (X) edge[out=-120,in=95,red] (h1)
(h1) edge[out =-85, in=175,red] (h2)
(h2) edge[->, out =-5, in=-155,red] (U)
;

\edge {X}{Dp}
\edge {D}{Yp}
\edge {D}{Dp}

\path (Qp) edge[->, bend right=12] (Y);

\edge {X}{Up}
\path (Qp) edge[->,out=-135,in=135] (U);

\edge[red] {X}{D};
\edge[red] {D}{Y};
\edge[red] {Y}{Xp};
\draw[->, red] ([yshift=-0.5mm]Xp.east) -- ([yshift=-0.5mm]Dp.west);
\edge[red] {Dp}{U};

\draw[->, blue] ([yshift=0.5mm]Xp.east) -- 
([yshift=0.5mm]Dp.west);

\edge[blue] {Yp}{Xp};
\edge[blue] {Qp}{Yp};
\edge[blue] {Qp}{Up};
\edge[blue] {Dp}{Up};
  
\end{influence-diagram}
\caption{Copying $Y$ and $U$ ensures \systemsAndPathsUniquenessForDef  %
}\label{fig:transform-2}
\end{subfigure}\hspace{5mm}%
\begin{subfigure}[t]{0.22\textwidth}
\centering
\begin{influence-diagram}
  \setcompactsize[node distance=0.5cm]
  \node (X) [] {$X^s$};
  \node (D) [below = of X,decision] {$D^s$};
  \node (Y) [below = of D] {$Y$};
  \node (Xp) [right = 4mm of Y] {$X'$};
  \node (Dp) [right = 4mm of Xp,decision] {$D'$};
  \node (U) [right = of Dp,utility] {$U$};
  \node (O)  at (Dp|-D) {$O$};
  \node (Yp) at (Dp|-X) {$Y'$};
  \node (Up) at (U|-D) [utility] {$U'$};
  \node (Qp) at (U|-X) [] {$Q'$};
  \node (h1) [draw=none,left = 1.5mm of Y, inner sep=0mm, minimum size=0mm] {};
  \node (h2) [draw=none,below = 1.5mm of Y, inner sep=0mm, minimum size=0mm] {};
\path (X) edge[out=-120,in=95,red] (h1)
(h1) edge[out =-85, in=175,red] (h2)
(h2) edge[->, out =-5, in=-155,red] (U)
;

\edge {X}{Up}
\edge {Yp}{Xp}
\path (Qp) edge[->,out=-135,in=135] (U);

\edge{X}{Dp}
\edge{D}{Yp}
\edge{D}{Dp}
\path (Qp) edge[->, bend right=12] (Y);
\edge {O}{Dp}
\edge{Y}{O}

\edge[red] {X}{D};
\edge[red] {D}{Y};
\edge[red] {Y}{Xp};
\draw[->, red] ([yshift=-0.5mm]Xp.east) -- ([yshift=-0.5mm]Dp.west);
\edge[red] {Dp}{U};

\draw[->, blue] ([yshift=0.5mm]Xp.east) -- ([yshift=0.5mm]Dp.west);
\edge[blue] {Xp}{O};
\edge[blue] {Yp}{O};
\edge[blue] {Qp}{Yp};
\edge[blue] {Qp}{Up};
\edge[blue] {Dp}{Up};
  
\end{influence-diagram}
\caption{Making a copy $O$ of $X'$, %
ensures no-backdoor-infopaths.~\looseness=-1
}\label{fig:transform-3}
\end{subfigure}\hspace{5mm}%
\begin{subfigure}[t]{0.22\textwidth}
\centering
\begin{influence-diagram}
  \setcompactsize[node distance=0.5cm]
  \node (X) [] {$X^s$};
  \node (D) [below = of X,decision] {$D^s$};
  \node (Y) [below = of D] {$Y$};
  \node (Xp) [right = 4mm of Y] {$X'$};
  \node (Dp) [right = 4mm of Xp,decision] {$D'$};
  \node (U) [right = of Dp,utility] {$U$};
  \node (O)  at (Dp|-D) {$O$};
  \node (Yp) at (Dp|-X) {$Y'$};
  \node (Up) at (U|-D) [utility] {$U'$};
  \node (Qp) at (U|-X) [] {$Q'$};
  \node (h1) [draw=none,left = 1.5mm of Y, inner sep=0mm, minimum size=0mm] {};
  \node (h2) [draw=none,below = 1.5mm of Y, inner sep=0mm, minimum size=0mm] {};
\path (X) edge[out=-120,in=95,red] (h1)
(h1) edge[out =-85, in=175,red] (h2)
(h2) edge[->, out =-5, in=-155,red] (U)
;

\edge{X}{Dp}
\edge{D}{Dp}
\edge {O}{Dp}

\edge[red] {X}{D};
\edge[red] {D}{Y};
\edge[red] {Y}{Xp};
\draw[->, red] ([yshift=-0.5mm]Xp.east) -- ([yshift=-0.5mm]Dp.west);
\edge[red] {Dp}{U};

\draw[->, blue] ([yshift=0.5mm]Xp.east) -- ([yshift=0.5mm]Dp.west);
\edge[blue] {Xp}{O};
\edge[blue] {Yp}{O};
\edge[blue] {Qp}{Yp};
\edge[blue] {Qp}{Up};
\edge[blue] {Dp}{Up};
  
\end{influence-diagram}
\caption{Finally, links are removed, ensuring \noRedundantLinks~\looseness=-1
}\label{fig:transform-4}
\end{subfigure}

\caption{\chris{Obtaining a normal form tree via homomorphic graph transformations:}
An ID graph (a)
is homomorphically transformed via graphs (b) and (c)
into a graph (d) whose tree is in normal form.
  }\label{fig:transforms} 
\end{figure*}
In order to generalise the approach taken to parameterising two systems, we need to reason
about the systems independently, in reverse order.
If the systems overlap, however, this 
makes it harder
to reason about them independently.
Thus it is useful to define a notion of systems called \emph{normal form} that are well-behaved.~\looseness=-1

\begin{restatable}[Normal form tree]{definition}{defnormalform} %
\label{def:sep5.3-normal-form-graph-with-tree}
A tree $T$ on $\sG$ is 
in
\emph{normal form} if all of the following hold:
\begin{enumerate}[label=(\alph*)]

    \item \label{def:sep5.3c-unique-systems-and-paths}
    (\systemsAndPathsUniquenessForDef) 
    A node $N$ in $T$ can only be in multiple paths $p^1,..., p^k$ of systems in the tree, if splitting $N$ into $\{N,N'\}$ via \cref{le:21may19.2-CID-hom-from-node-copying-and-deleting} and obtaining $T'$ from $T$ by replacing $N$ with $N'$ in one of those paths would make $T'$ no longer a tree of systems.

    \item (no-backdoor-infopaths) \label{def:sep5.3b-appropriate-tree}
    Every system $s$ in $T$ has an info path that starts with an outgoing link from $X^s$. 

    \item (\noRedundantLinks) \label{def:sep5.3e-no-redundant-links}
    If $N \to N'$ is an edge to a non-decision $N'$, where 
    one of $N$ and $N'$ is in a path in a system of $T$, 
    not including the nodes of the root information link, 
    then $N \to N'$ is in a path of a system of $T$.~\looseness=-1

\end{enumerate}
\end{restatable}
\Ryan{We should define front-door or just say ``starting with a tail''}

An arbitrarily chosen tree will not generally be in normal form. 
For example, \cref{fig:transform-1}
contains two systems (a red root system for $X^s \to D^s$
and a blue child system for $X' \to D'$) that constitute a tree, 
but this tree fails all three requirements for being in normal form.
However, by a series of homomorphic transformations,
it is possible to obtain a new graph with a tree of systems that is in normal form (as in \cref{fig:transform-4}).~\looseness=-1

\begin{restatable}[Normal Form Existence]{lemma}{lenormalformtransform}
\label{le:20nov29.1-Existence-of-adequate-CID-split-from-nothing}
Let $\sG$ be a soluble ID graph whose minimal $d$-reduction $\sG^*$ contains $X\to D$. 
Then there is a 
normal form tree $T'$
on a soluble ID graph $\sG'$,
with a homomorphism $h$ from $\calG'$ to $\calG$
where the information link $X' \to D'$ of the root system of $T'$,
has $h(X')=X$, $h(D')=D$,
and every node in $\sG$ is also in $\sG'$ 
but the only nodes in $\sG$ that are in $T'$ are $X$ and $D$.~\looseness=-1
\end{restatable}
\ryan{Is it actually meaningful for nodes in $\calG$ to be in $T'$, or do we 
need to talk about applying $h^{-1}$ to them first?}
\ryan{Need to state that $\calG^3=\calG',T^3=T',h'=\hfirsttolast$ in the proof.}
\ryan{Maybe remark at this point that it's interesting that 
we can homomorphically modify the tree to obtain one that behaves
differently, describing how this relates to the unidirectionality of the homomorphism property? / Address the apparent contradiction that homomorphisms preserved the 
properties of a tree, but that we can analyse them differently.}

Essentially, the procedure for obtaining a normal form tree proceeds in four steps:
\begin{enumerate}
    \item Construct a tree of systems on $X \to D$: First, pick any system for $X\to D$. Then, pick any system for every other information link $X' \to D'$ in the existing system. Iterate until every link in the tree has a system. 
    
    \item 
    Make a copy (lemma 12) of each node for each position (basically, each path) that node has in the tree. 
    This ensures position-in-tree-uniqueness.
    
    \item For systems whose infopath starts with an incoming link $X \gets Y$, copy $X$ (lemma 12), to obtain $X\to  O \gets Y$. 
    This ensures no-backdoor-infopaths.
    
    \item Prune the graph (using lemma 13), by removing any (non-information) links outside the tree of systems. 
    This ensures no-redundant-links.
\end{enumerate}
For example, in \cref{fig:transform-1,fig:transform-2,fig:transform-3,fig:transform-4},
three transformations are performed, each of which makes the tree meet one additional 
requirement, ultimately yielding a normal form tree (\cref{fig:transform-4}) with a homomorphism to the original.~\looseness=-1

\subsection{Proving positive VoI given a normal form tree} \label{sec:voi-given-nft-main}

The reason for using normal form trees is that they enable each system to be parameterized and solved independently.
In particular, we know that the optimal policy for one system involves reproducing information from ancestor nodes such as $Q^s$.
As optimal policies can be found with backwards induction in soluble graphs, our approach involves finding optimal policies in reverse order.
It will therefore suffice to prove that 
non-descendant systems cannot provide information about ancestor nodes within the system.
For example, 
in \cref{fig:transform-1}, when solving for $\pi^{D'}$, 
we would like to know that $D^s$ cannot provide information about $Q'$.~\looseness=-1

\newcommand{\ObsDesc}{\bf{{ObsDesc}}}
\newcommand{\Back}{\bf{Back}}

\newcommand{\graphknowledgelemmastatement}{
Let $s$ be a system in a normal form tree $\mathcal{T}$ on a soluble ID graph $\calG$.
Let $\Pa^{-s}=\Pa(D^s) \setminus \sV^s$ be $D^s$'s out-of-system parents, $\Pa^{s}= {\Pa(D^s)\cap \sV^s}$ 
be the within-system parents of $D^s$,
$\ObsDesc^s$ be the 
observation nodes in descendant systems of $s$, and let $\Back^s=\sV^s \cup (\Anc(D^s) \setminus \Fa(D^s))$. %
Then \mbox{$\Back^s \dsep \Pa^{-s}\setminus \ObsDesc^s \mid \Pa^s\cup \ObsDesc^s$}.
}

\begin{restatable}[\GraphKnowledgeLemmaName]{lemma}{graphKnowledgeLemma} \label{le:2v2-graph-knowledge-lemma}
\graphknowledgelemmastatement
\end{restatable}

For example, \cref{fig:transform-4}, has a normal form tree, 
which implies the assurance that $X'$ cannot use information from the red system 
to tell it about $Q'$; formally, $Q' \dsep (Y \cup X^s) \mid X'$. %
Given that each decision $D^s$ in the tree cannot use information from ancestor systems, we can then 
prove that $D^s$
cannot know enough about $X^s$ and $Q^s$ to 
perform optimally, without observing $X^s$.
More formally: ~\looseness=-1

\begin{restatable}[VoI Given Normal Form Tree]{lemma}{materialitysat}  \label{le:aug23.3v2-materiality-model-on-nf-tree-has-materiality}
Let $\sG$ be a soluble ID graph with a normal form tree with root info link $X\to D$. Then there exists an 
ID
compatible with $\sG$ for which $X$ has positive VoI for $D$.~\looseness=-1
\end{restatable}

The formal proof is given in \cref{sec:model-has-incentives}.
Informally, in order to show that the decision of each system is forced to behave as intended despite there now being a tree of systems full of other decisions, we use \cref{le:2v2-graph-knowledge-lemma} to show that the utility that a decision obtains in system $s$ only depends on the information it obtains from within system $s$. This rules out that ancestor decisions can observe and pass along relevant information via a path outside the system. 
Moreover, we know by the solubility assumption that the optimal decision rule at a later decision 
cannot depend on the decision rule followed by earlier decisions.
The argument then proceeds by backward induction.
The final decision $D^{s^n}$ must copy the value of $X^{s_n}$.
Given that it does so, the penultimate decision $D^{s^{n-1}}$ must do the same.
And so on, until we find that 
$D$ must copy $X$, and cannot do so in any way other than by observing it, 
meaning that $X$ has positive VoI for $D$.~\looseness=-1\chris{I think these explanatory paragraphs can be better.}
\ryan{Is this accurate \& better?}

Finally, we can prove our main result, that there exists an ID on $\calG$ where $X$ has positive VoI.

\begin{proof}[Proof of {\cref{thm:voi}} (completeness direction)]
We know that the d-reduction $\calG^*$ of $\calG$ contains $X \to D$.
By \cref{le:20nov29.1-Existence-of-adequate-CID-split-from-nothing}, there exists an ID graph $\calG'$ 
with normal form tree rooted at a link $X' \to D'$, with an ID homomorphism from $\calG'$ to $\calG$ that has $h(X')=X$ and $h(D')=D$. 
By \cref{le:aug23.3v2-materiality-model-on-nf-tree-has-materiality}, since $\calG'$ has a normal form tree rooted at $X'\to D'$, there exists an ID on $\calG'$ in which $X'$ has positive VoI for $D'$.
By \cref{th:CID-homomorphism-preserves-Materiality}, the presence of the ID homomorphism $h$ from $\calG'$ to $\calG$
means that there also exists an ID $\calM$ on $\calG$ such that $h(X')=X$ has positive VoI for $h(D')=D$, showing the result.~\looseness=-1
\end{proof}

\section{Applications \& Implementation} \label{sec:applications}
\ryan{Change the figure to all-superscripts?}
\ryan{Cite frameworks paper.}
\ryan{Illustrate nodes with +VoC?}

Graphical criteria can help with modeling agents' incentives in a 
wide range of settings including 
(factored) Partially Observed Markov Decision Processes (POMDPs)
and Modified-action Markov Decision Processes \citep{langlois2021rl}.
For concreteness, 
we show how our contributions can aid in analysing a supervision POMDP \citep{Milli2017}.
In a supervision POMDP, an AI interacts with its environment, given suggested actions from a 
human player.
We will assume that the human's policy has already been selected, in order to focus on 
the incentives of the AI system.~\looseness=-1

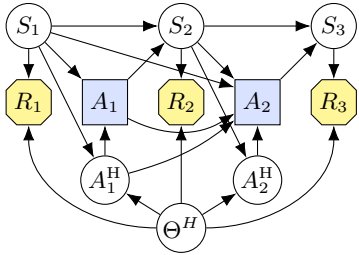
\begin{figure}[H]
\vspace{-1mm}
\begin{center}
\setcompactsize[node distance=0.4cm]
    \begin{tikzpicture}[
        every node/.style={
            draw, circle, minimum size=0.6cm, inner sep=0.3mm}]
        \node (R1) [utility] {$R_1$};
        \node (S1) [above = of R1] {$S_1$};
        \node (A1) [right = of R1, decision] {$A_1$};
        \node (R2) [right = of A1, utility] {$R_2$};
        \node (S2) [above = of R2] {$S_2$};
        \node (A2) [right = of R2, decision] {$A_2$};
        \node (R3) [right = of A2, utility] {$R_3$};
        \node (S3) [above = of R3] {$S_3$};
        
        \edge {S1} {R1};
        \edge {S2} {R2};
        \edge {S3} {R3};
        \edge {A1,S1} {S2};
        \edge {A2,S2} {S3};
        \edge {S1} {A1};
        \edge {S2} {A2};
        
        \node (A1h) [below = of A1] {$A^\mathrm{H}_1$};
        \node (A2h) [below = of A2] {$A^\mathrm{H}_2$};
        
        \node (help) [below = 0mm of A1h, draw=none] {};
        
        \node (theta) at (R2 |- help) [] {$\Theta^H$};
        
        \node (space) [minimum size=0mm, node distance=2mm, below left = 1em of A1h, draw=none] {};
          
        \path 
        (A1h) edge[->] (A1)
        (A1h) edge[->, bend right=11] (A2)        
        (A2h) edge[->] (A2)
        
        (A1) edge[->, bend right] ([yshift=-1.5mm]A2.west)
        
        (S1) edge[->] (A1h)
        (S2) edge[->] (A2h)
        (theta) edge[->] (A1h)
        (theta) edge[->] (A2h)        
        
        (theta) edge[->, out=180, in=270] (R1)
        (theta) edge[->] (R2)
        (theta) edge[->, out=0, in=-90] (R3)
        ;
        \draw[->] (S1) -- ([yshift=2mm,bend left=3]A2.west);
        
      \node (ah) [minimum size=0mm,node distance=2mm, below left = of R1, draw=none] {};
    \end{tikzpicture}
\end{center}\vspace{-1mm}
\caption{A supervision POMDP with the human considered part of the environment;
we show 3 timesteps and 2 actions.~\looseness=-1
}\label{fig:spomdp}
\end{figure}

Given the graph in \cref{fig:spomdp}, we can apply the VoI criterion to each $A^H_i$, the sole parent of $A^i$.
The minimal d-reduction is identical to the original graph, 
so since $A^H_i \!\not \dsep\! R^{i+1} \!\mid\! \emptyset$, 
the observation $A^H_i$ can have positive VoI.
This formalises the claim of \citet{Milli2017} that in a supervision POMDP, the agent ``can learn about reward through [the human's] orders''. 
We can say the same about Cooperative Inverse Reinforcement Learning (CIRL).
CIRL differs from supervision POMDPs only in that each human action 
$A^H_i$
directly 
affects the state $S_{i+1}$. 
If $\calG$ is modified by adding edges $A^H_i \to S_{i+1}$,
and the VoI criterion is applied at $A^H_i$ once again,
we find that 
$A^H_i$ may have positive VoI for $A^i$, 
thereby formalising
the claim that the robot is
``incentivised to learn''
\citep[Remark 1]{Hadfield-Menell2016cirl}.~\looseness=-1

To facilitate convenient use of the graphical criterion, 
we have implemented it in the open source ID library \emph{pycid} \citep{Fox2021pycid}, 
whereas the previous implementation was limited to single-decision IDs.%
\footnote{Code is available at {\url{www.github.com/causalincentives/pycid.}}}~\looseness=-1

\section{Related Work} \label{sec:related-work}

\newcommand{\aref}[1]{\hyperref[#1]{Appendix~\ref{#1}}} %

\paragraph{Value of information}
The concept of value of information dates back to the earliest
papers on influence diagrams \citep{Howard1966,matheson1968economic}.
For a review of recent advances, 
see \citet{borgonovo2016sensitivity}.~\looseness=-1

Previous results have shown how to identify 
observations with zero VoI or equivalent properties
in various settings.
In the no forgetting setting, 
\citet{Fagiuoli1998} and \citet{Nielsen1999}
identified ``structurally redundant''
and ``required nodes'' respectively.
In soluble IDs,
\citet{nilsson2000evaluating} 
proved that optimal decisions need not 
rely on nonrequisite nodes.
Completeness proofs in a setting of one decision have been discovered for VoI and its analogues by
\citet{zhang2020causal,lee2020characterizing,Everitt2021agent}.
Finally, in insoluble IDs,
\citet{lee2020characterizing}
proved that certain nodes are
``redundant under optimality''.
Of these works, only \citet{Nielsen1999} attempts a completeness result for the multi-decision setting.
However, as pointed out by \citet{Everitt2021agent}, it falls short in two respects:
Firstly, the criterion $X \not \dsep \sU^D \mid \Pa(D)$ is proposed, which differs from 
nonrequisiteness
in the conditioning set.
Secondly, and more importantly, 
the proof is incomplete because it 
assumes that positive VoI follows from d-connectedness.%

\paragraph{Submodel-trees}
Trees of systems are loosely related to the ``submodel-trees'' of \citet{lee2021submodel}.
In both cases, the tree encodes an ordering in which the ID can be solved, so the edges in a tree of systems are analogous to those in a submodel-tree. 
The nodes, however, (i.e.\ systems and submodels) differ.
Whereas a submodel-tree aids with solving IDs, a tree of systems helps with parameterising an ID graph. 
As a result, a submodel contains all nodes relevant for $D$, whereas a system consists just one set of info-/control-/obs-paths. 
Relatedly, in a submodel, downstream decisions may be solved and replaced with a value node, whereas in a tree of systems, they are not.~\looseness=-1

\section{Discussion and Conclusion} \label{sec:discussion}

This paper has described techniques for analyzing soluble influence diagrams.
In particular, we introduced ID homomorphisms, a method for transforming ID graphs while preserving key properties, and showed how these can be used to establish equivalent ID graphs with conveniently parameterizable ``trees of systems''.
These techniques enabled us to derive the first completeness result for a graphical criterion for value of information in the multi-decision setting.

Given the promise of reinforcement learning methods, it is essential that we obtain a formal understanding of how multi-decision behavior is shaped.
The graphical perspective taken in this paper has both advantages and disadvantages. On the one hand, some properties cannot be 
distinguished from a graphical perspective alone. On the other hand, it means our results are applicable even when the precise relationships are unspecified or unknown. 
There are a range of ways that this work could be beneficial.
For example, analogous results for the single-decision setting have contributed to safety and fairness analyses
\citep{Armstrong2020pitfalls,cohen2020unambitious,everitt2019tampering,Everitt2019modeling,langlois2021rl,Everitt2021agent}.~\looseness=-1

Future work could include applying the tools developed in this paper to other incentive concepts such as value of control \citep{shachter1986evaluating}, instrumental control incentives, and response incentives \citep{Everitt2021agent}, 
to further analyse the value of remembering past decisions
\citep{Shachter2016,lee2020characterizing},
and to generalize the analysis to multi-agent influence diagrams \citep{Hammond2021equilibrium,Koller2003}.

\section*{Acknowledgments}
This work was supported in-part by the Leverhulme Centre for the Future of Intelligence, Leverhulme Trust, under
Grant RC2015-067.

\bibliography{library.bib}

\onecolumn
\appendix

\section{Background for the proofs} \label{sec:background}
We review two properties of IDs --- and all Bayesian networks --- that we will use throughout our proofs.
\begin{lemma}[Active paths between ancestors contain only ancestors]
\label{le:2.12-active-paths-between-ancestors-contain-only-ancestors}
A path from an ancestor of node $N$ to another ancestor of $N$, that is active conditional on only ancestors of $N$, contains only ancestors of $N$.~\looseness=-1
\end{lemma}
\begin{proof}
Let $p\colon  A \upathto B$ be any active path where $A$ and $B$ are ancestors of $N$ and let the colliders on $p$ be $\sO$.
Since $p$ is active, any collider $O \in \sO$ on $p$ is an ancestor of the conditioned set, and is therefore an ancestor of $N$.
For any chain or fork node $V$, choose one of its outgoing edges along $p$ and follow $p$ until the next collider or endpoint ($\sO$, $A$, or $B$).
This path is directed, so $V$ is an ancestor of $O$, $A$, or $B$, and hence $M$.
\end{proof}

The standard definition of a walk is a sequence of consecutive edges. 
Unless a node has an edge to itself, a walk is not alowed to visit the same node twice in a row.
Instead, we define a notion of a walk such that it is always allowed to repeat the same node previously visited.

\begin{restatable} [Walk with node repetition]{definition}{defwalkwithnoderepetition} \label{def:21feb17.1-walk-with-node-repetition}
A \emph{walk with node repetition} from $N^1$ to $N^n$ in a graph $(\sV, \sE)$ is a sequence of nodes $N^1,...,N^n$ such that for any $i\in \{1,...,n-1\}$, either there is a link $N^i\to N^{i+1}$ or $N^i\gets N^{i+1}$ in $\sE$, or $N^i=N^{i+1}$.
\end{restatable}

We say that a node $N$ in a walk with node repetitions $p$ is a collider/fork/chain node in a walk with node repetitions $w$ if it is a collider/fork/chain node in the walk (without node repetitions) $w'$ obtained by removing any nodes that are equal to their predecessor.

\begin{lemma} [``Active" walk with node repetitions implies active path] \label{21jan21.2-excising-loops-from-active-walks-to-get-active-paths}
If there is a walk with node repetition
from node $A$ to node $B$, such that all fork and chain nodes are not in a set $\sZ$, and all collider nodes have a descendant in $\sZ$, then there is a path between $A$ and $B$ that is active given $\sZ$.
\end{lemma}
\newcommand{\walk}{w}
\newcommand{\ppath}{p}
\begin{proof}
Assume there is a walk with node repetition $\walk\colon A \upathto B$ such that every collider in $\walk$ has a descendant in $\sZ$ and every non-collider in $w$ is not in $\sZ$. 
Then let $p$ be the path obtained from $\walk$ 
by replacing every segment $N \upathto N$ with the node $N$.
Clearly, $\ppath$ is a path, so we will proceed to show that it is active given $\sZ$,
by showing that it is active at each of its nodes.

Assume that $N$ is a collider in $p$.
Then, $N$ was obtained from 
a segment in $w$, $Y_1 \to N \upathto N \leftarrow Y_2$
where $N \upathto N$ has length zero or greater.
For this segment to be active in $w$,
the first collider in $N \upathto N$ must have a descendant in $\sZ$, 
and thus so does $N$, and it is active in $\ppath$.
Assume instead that $N$ is a non-collider.
Then, $N$ was obtained from 
a segment in $w$, $Y_1 \to N \upathto N \to Y_2$, 
$Y_1 \leftarrow N \upathto N \leftarrow Y_2$, or 
$Y_1 \leftarrow N \upathto N \to Y_2$.
In any case, for this segment to be active in $w$, 
$N \not \in \sZ$, so it is active in $\ppath$, proving the result.
\end{proof}

\begin{figure}
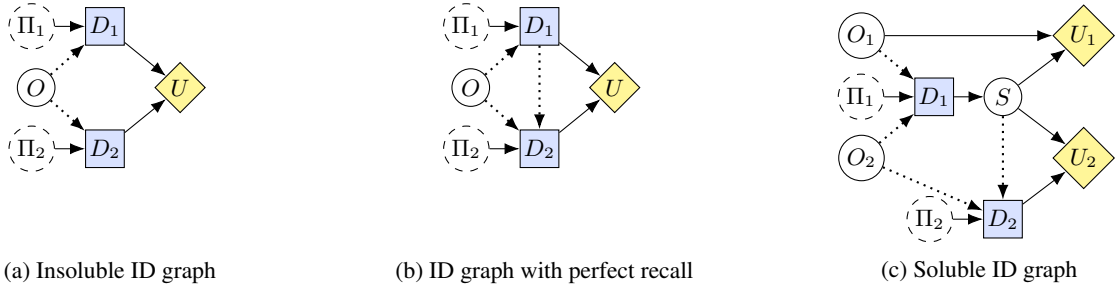

  \centering
  \begin{subfigure}{0.32\textwidth}
    \centering
    \begin{influence-diagram}
      \setcircularnodes
      \setcompactsize
      \node (h) [draw=none] {};
      \node (O) [left = of h] {$O$};
      \node (D1) [above = of h, decision] {$D_1$};
      \node (D2) [below = of h, decision] {$D_2$};
      \node (U) [right = of h, utility] {$U$};
      \node (Pi1) [policy,left = of D1] {$\Pi_1$};
      \node (Pi2) [policy,left = of D2] {$\Pi_2$};
      \node (phantom) [draw=none,below= 5mm of D2] {};

      \edge {D1,D2} {U};
      \edge{Pi1}{D1};
      \edge{Pi2}{D2};
      \edge[information] {O} {D1,D2};

    \end{influence-diagram}
    \caption{Insoluble ID graph}
    \label{fig:no-bi}
  \end{subfigure}
  \begin{subfigure}{0.32\textwidth}
    \centering
    \begin{influence-diagram}
      \setcircularnodes
      \setcompactsize
      \node (h) [draw=none] {};
      \node (O) [left = of h] {$O$};
      \node (D1) [above = of h, decision] {$D_1$};
      \node (D2) [below = of h, decision] {$D_2$};
      \node (U) [right = of h, utility] {$U$};
      \node (Pi1) [policy,left = of D1] {$\Pi_1$};
      \node (Pi2) [policy,left = of D2] {$\Pi_2$};
      \node (phantom) [draw=none,below= 5mm of D2] {};

      \edge{Pi1}{D1};
      \edge{Pi2}{D2};
      \edge {D1,D2} {U};
      \edge[information] {O} {D1};
      \edge[information] {O, D1} {D2}
      
    \end{influence-diagram}
    \caption{ID graph with perfect recall}
    \label{fig:perfect-recall}
  \end{subfigure}
  \begin{subfigure}{0.32\textwidth}
    \centering
    \begin{influence-diagram}
      \setcircularnodes
      \setcompactsize
      \node (h) [draw=none] {};
      \node (h2) [right = of h, draw=none] {};
      \node (O) [left = of h] {$O_2$};
      \node (D1) [above = of h, decision] {$D_1$};
      \node (S) [above = of h2] {$S$};
      \node (h1) [above = of D1, draw=none] {};
      \node (O1) [left = of h1] {$O_1$};
      \node (D2) [below = of h2, decision] {$D_2$};
      \node (U) [right = of h2, utility] {$U_2$};
      \node (U1) at (U|-O1) [utility] {$U_1$};
      \node (Pi1) [policy,left = of D1] {$\Pi_1$};
      \node (Pi2) [policy,left = of D2] {$\Pi_2$};

      \edge{Pi1}{D1};
      \edge{Pi2}{D2};
      \edge {D1} {S};
      \edge {S,O1} {U1};
      \edge {S,D2} {U};
      \edge[information] {O,O1} {D1};
      \edge[information] {O,S} {D2};
     
    \end{influence-diagram}
    \caption{Soluble ID graph}
    \label{fig:sufficient-recall}
  \end{subfigure}
  \caption{Multi-decision IDs.}
  \label{fig:multi-decision}
\end{figure}

For our analysis, we consider soluble ID graphs.
This condition includes graphs with perfect recall \cref{fig:perfect-recall}
but also includes some others, shown in \cref{fig:sufficient-recall}.

\section{ID homomorphisms} \label{sec:CID-homomorphisms}
\subsection{Properties preserved given an ID homomorphism} \label{sec:properties-preserved-proofs}

We now prove two
properties that are preserved by any ID homomorphism:\footnote{An ID homomorphism is analogous to the notion of graph homomorphism from graph theory, which essentially requires that edges are preserved along the map. 
In fact, if we would consider every node in an ID graph as a decision and as having an edge to itself, then any ID homomorphism is also a graph homomorphism when considering the two ID graphs as ordinary graphs (ignoring node types).}
solubility, and VoI.~\looseness=-1

\lecidhomsufficientrecall*

\begin{proof}
Since $\sG$ is soluble, there is a total ordering of decisions $<$  such that for all $D^1<D^2$, $\Pi^{D} \dsep U(D^2) \mid \Fa(D^2)$. 
To show that $\sG'$ is also soluble, we use $<$ to construct an ordering on the decisions of $\sG'$ that has the same property. 
We define $<'$ for decisions in $\sG'$: as $D^1 <' D^2$ when:~\looseness=-1
\begin{itemize}
    \item $h(D^1) \neq h(D^2)$ and $h(D^1) < h({D^2})$; or
    \item $h(D^2) = h(D^2)$ and $\sG'$ contains $D^1 \to D^2$.
\end{itemize}
This is a total order since whenever $h(D^1) \neq h(D^2)$ then either $h(D^1) < h(D^2)$ or $h(D^1) > h(D^2)$ by the total order on $\sG$, and whenever $h(D^1) = h(D^2)$ then $\sG'$ contains $D^1 \to D^2$ or $D^1 \leftarrow D^2$ by (\combinesonlylinkeddecisions).

\newcommand{\porig}{p_\sorig}
\newcommand{\porigwalk}{w}
Now we show that for any two decisions $D^1,D^2$ in $\sG'$, with $D^1<'D^2$, that any path $p\colon \Pi^{D^1} \to D^1 \upathto U$ for some $U \in \Desc (D^2)$ cannot be active given $\Fa(D^2)$. Consider two cases:

	\textit{Case (1) }: Assume $h(D^2)\!=\!h(D^1)$. Then $D^1\!<\!D^2$ so $\calG'$ contains $D^1 \to D^2$ by definition of $<'$, 
	and so any path $p\colon \Pi^{D^1} \to D^1 \upathto U$ that starts with the link $\Pi^{D^1} \to D^1 \to Y$ is blocked at $D^1$ given $\Pa({D^2})$.
	Any path that begins as $\Pi^{D^1} \to D^1 \leftarrow Y$ is blocked at the non-collider $Y$:
	the presence of $D^1 \leftarrow Y$ implies 
	that $\sG$ contains $h(D^1) \leftarrow h(Y)$ (\preserveslinks),
	so that $\sG'$ contains $D^2 \leftarrow Y$ (\coversallinfolinks), 
	and $Y \in \Pa(D^2)$.

	\textit{Case (2)} : Assume $h(D^2)\neq h(D^1)$. 
	We will prove the contrapositive:
	if $\Pi^{D^1} \not\dsep U(D^2) \mid \Fa(D^2)$ in $\sG'$
	then $\Pi^{h(D^1)} \not\dsep U(h(D^2)) \mid \Fa(h(D^2))$ in $\sG$ where $h(D^1)<h(D^2)$.
	If $p\colon\Pi^{D^1} \to D^1 \upathto U(D^2)$ is active given $\Fa(D^2)$, then consider the walk with node repetition $\porigwalk\colon\Pi^{h(D^1)} \to h(D^1) \upathto h(U(D'))$ consisting of 
	$f(V)$ for each node $V$ in $p$. 
	We know that each $V$ in $p$ is a (chain/fork/collider) if and only if 
	$h(V)$ is a (chain/fork/collider) in $\porigwalk$, since if there is a link $N\to V$ or $N\gets V$ then there must be a link $h(N)\to h(V)$ or $h(N) \gets h(V)$ respectively by the (\preserveslinks) assumption of ID homomorphisms. 
	And that each node $V$ contains a descendant in $\Fa(D^2)$
	if and only if $h(V)$ contains a descendant in $\Fa(h(D^2))$.
	So every collider in $w$ has a descendant in $\Fa(h(D^2))$ while every non-collider does not.
	This implies that $\Pi^{h(D^1)} \not \dsep U(h(D^2)) \mid \Fa(h(D^2))$ by \cref{21jan21.2-excising-loops-from-active-walks-to-get-active-paths}, 
	and we know that $h(D^1)<h(D^2)$ by the definition of $<'$
	so the result follows.
\end{proof}

We can now define how a homomorphism allows us to define a procedure for 
transporting IDs between the two graphs, such that corresponding 
IDs and policies lead to equivalent outcomes.

\lecidhomequivalence*

\begin{proof}
We define the \emph{ported ID} $\calM=(\sG,\dom,P)$ as follows: Each non-utility node $N$ in $\sG$ has $\dom(N) =  \prod_{N^i\in h^{-1}(N)} \dom(N^i)$, and each utility node $U$ in $\sG$ has $\dom (U) = \mathbb R$. Each non-decision node $N$ has as $P^N(n|\pa)$ the joint conditional distribution of each $\tilde P^{N^i}(n^i|\pa^i)$. We define the \emph{ported policy} $\pi$ so that each decision $D$ has as $\pi^D(d|\pa)$ the joint conditional distribution of each $\pi^{D^i}(d^i|\pa)$. These in fact factor over $\sG$  by property (b) of ID homomorphisms.

We show the result by induction on the graph of nodes $N^i$ in $\Gsplit$. Let $N=\sorig(N^i)$. 

\sloppy \textit{base step} : \textit{Assume $N^i$ has no parents in $\Gsplit$}. Then $P(N\!=\!(n^1,...,n^k)) =P^N((n^1,...,n^k))=\prod_{i=1}^k \tilde P^{N^i}(n^i|n^1,...,n^{i-1})
=\tilde P(N^1\!=\!n^1,...,N^k=n^k)$.

\textit{inductive step} : \textit{Assume that for all parents $Y^i$ of $N^i$, letting $Y=\sorig(Y^i)$, we have that $P(Y\!=\!(y^1,...,y^k))=\tilde P(Y^1\!=\!y^1,...,Y^k=y^k)$}. Then
\begin{alignat*}{2}
    &P(N\!=\!(n^1,...,n^k))\\
    &=  \sum_{y^1,...,y^k}P(Y\!=\!(y^1,...,y^k))\cdot P^N((n^1,...,n^k)|y^1,...,y^k)
    \quad\quad &&{}\\
    &=\sum_{y^1,...,y^k}\tilde P(Y^1\!=\!y^1,...,Y^k\!=\!y^k))\cdot \prod_{i=1}^k \tilde P^{N^i}(n^i|n^1,...,n^{i-1}, y^1,...,y^k)
    \quad\quad && \text{}\\
    &=\tilde P(N^1\!=\!n^1,...,N^k=n^k) \quad &&\text{} 
\end{alignat*}
Which shows the result.
\end{proof}

We will write 
the ``transported ID'' and policy from \cref{le:cidhom1-equivalence} as
$h(M)$ and $h(\pi)$. This means that we also treat a homomorphism $h$ as a function between IDs and policies. In fact, in order to show that ID homomorphisms preserve VoI %
, we show that on policies, $h$ is a bijection, which relies on the properties (c) and (d) of ID homomorphisms,  and is the primary reason why (c,d) are included:
\begin{restatable}{lemma}{lecidhompolicybijection} \label{le:bijection-between-deterministic-policies-on-split-graph}
Any ID homomorphism $h$ is a bijection from (optimal) policies on $\tilde \calM$ to (optimal) policies on $\calM=h(\tilde \calM)$. 
\end{restatable}

\begin{proof}
We define an inverse for the map as follows: Take a policy $\pi_D$ on $M$.
This gives a joint distribution $\tilde \pi^{D^i}=(\pi^{D})^i$ over $\dom(D^i)$ for $D^i\in \mathcal D=h^{-1}(D)$. Moreover, for any $X^j\in \Pa(D)$ and for any $X^{j,k}\in h^{-1}(X^j)$,  each decision $D^i$ has $X^{j,k}\in \Pa(D^i)$  by (\coversallinfolinks), and since these decisions $D^i$ form a complete graph (each $D^i$ is linked to each $D^j$) 
by condition (\combinesonlylinkeddecisions), this distribution $\tilde \pi^{D^i}$ also factors over $\tilde \sG$ and hence is a policy $\tilde M$. But this is precisely the definition of $\pi^D$ being the transported policy of $\tilde \pi^{\mathcal D}$, so that $\pi^D \mapsto \pi^{\mathcal D}$ is indeed the desired inverse. The optimal version of this lemma then follows from \cref{le:cidhom1-equivalence}.
\end{proof}

\lecidhompreservesmateriality*

\begin{proof}
\newcommand{\MatValue}[2]{\EE^{#2}_{#1}(\totutilvar)}
\newcommand{\maxMatValueExpanded}[3]{\max\limits_{#1}\MatValue{#2}{#3}}
\newcommand{\maxMatValueInlineExpanded}[3]{\max_{#1}\MatValue{#2}{#3}}
\newcommand{\maxMatValue}[2]{\maxMatValueExpanded{#1}{#1}{#2}}
\newcommand{\maxMatValueInline}[2]{\maxMatValueInlineExpanded{#1}{#1}{#2}}
Let $\mathcal X$ be the set of nodes $X^j$ in $\sG'$ such that $h(X^j)=X$, and $\mathcal D$ the set of nodes $D^j$ such that $\sorig(D^j)=D$. 

Firstly, note that $\sorig$ is also a homomorphism from  $\sG'_{\mathcal X\to \mathcal D}$ to $\sG_{X\to D}$ and from  $\sG'_{\mathcal X\not\to \mathcal D}$ to $\sG_{X\not\to D}$ (since in both cases, there is still an edge $X^i\to D^i$ iff there is an edge $X\to D$). Hence, for any policy $\pi'$ on $M'$ and letting $\pi=h(\pi')$ be the corresponding policy on $M$, apply \cref{le:cidhom1-equivalence} twice to conclude that 
$\MatValue{\pi}{M_{X\to D}}=\MatValue{\pi'}{M'_{\mathcal X\to \mathcal D}}$
and 
$\MatValue{\pi'}{M'_{X\not\to D}}=\MatValue{\pi}{M_{\mathcal X\not\to \mathcal D}}$.

Since the map that maps a policy on $M'$ to the corresponding policy on $M$ (see \cref{le:cidhom1-equivalence}) is a bijection by  \cref{le:bijection-between-deterministic-policies-on-split-graph}, this implies that 
$\maxMatValueInline{\pi}{M_{X\to D}}=\maxMatValueInline{\pi'}{M'_{\mathcal X\to \mathcal D}}$
and 
$\maxMatValueInline{\pi}{M_{X\not\to D}}=\maxMatValueInline{\pi'}{M'_{\mathcal X\not\to \mathcal D}}$.
These imply:
\begin{alignat*}{2}
    \maxMatValue{\pi}{M_{X\to D}} 
    &= \maxMatValue{\pi'}{M'_{\mathcal X\to \mathcal D}} \quad\quad &&: \text{by the argument above}\\
    &\geq \maxMatValue{\pi'}{M'_{X^i\to D^i}} \quad\quad &&: \text{more infolinks cannot decrease max utility}\\
    &> \maxMatValue{\pi'}{M'_{X^i\not\to D^i}} \quad &&:\text{by assumption: $X^i$ is material for $D^i$} \\
    &\geq 
    \maxMatValue{\pi'}{M'_{\mathcal X\not\to \mathcal D}} \quad &&:\text{more infolinks cannot decrease max utility} \\
    &= 
    \maxMatValue{\pi}{M_{X\not\to D}} \quad &&:\text{by the argument above}
\end{alignat*}
which shows the result.
\end{proof}

\subsection{Transformations that ensure a homomorphism} \label{app:ensure-cidhom}
\lecidhomcomposition*%

\begin{proof} We show that each of the four properties is preserved under composition:

(a) If $h$ and $h’$ preserve node types, then clearly so does $h \circ h'$. 

(b) If $\sG''$ contains $A \to B$ then by (b) for $h'$, 
$\sG'$ contains $h’(A)\to h'(B)$ or $h'(A)=h'(B)$. 
In either case, (b) for $h$ implies that
$\sG$ contains $h\circ h'(A) \to h\circ h'(B)$ or $h\circ h'(A) = h\circ h'(B)$.

(c) If $\sG$ contains $h\circ h'(N) \to h\circ h'(D)$, then by (c) for $h$, $\sG'$ contains $h’(N) \to h'(D)$ and by the same argument 
$\sG$ contains $N\to D$.

(d) Assume $h\circ h'(D^1)=h\circ h'(D^2)$ and $D^1\!\neq\! D^2$ in $\sG''$. 
Then if $h’(D^1)\!=\!h'(D^2)$, by (d) for $h'$, 
$\sG''$ contains
$D^1\!\to\! D^2$ or $D^2\!\to\! D^1$ showing the result. 
If $h’(D^1)\neq h'(D^2)$ then by (d) for $h$, 
$\sG'$ contains $h’(D^1)\to h'(D^2)$ or $h’(D^2) \to h'(D^1)$, and hence by (c) for $h'$, $\sG''$ contains $D^1 \to D^2$ or $D^2\to D^1$.
\end{proof}

\lecopyingcidhom*
\ryan{I've simplified/shortened this a bit further. Feel free to revert any changes is preferred.}

\begin{proof}
ID homomorphism condition (a) follows by definition. (b) follows from the definition of $E'$. (c,d) follow since they hold for all nodes $N$ by definition, including the decisions.
\end{proof}

\lepruningcidhom*

\begin{proof}
The homomorphism properties follow:
(a) by definition. (b) from $E'\!\subseteq\! E$, (c) from every $N\!\to\! D \in \sD$ 
being in $E$, (d) from $h$ being the identity map so every $D^1\!\neq\! D^2$ has $h(D^1)\!\neq\! h(D^2)$.~\looseness=-1
\end{proof}%

\section{Systems and trees of systems in an ID graph} \label{sec:preliminaries-systems-and-trees}

\subsection{Systems}

Before detailing the properties of systems, we first recap the elements of a system.
We call $D^s$, $U^s$, $X^s$, and $\sinfolink^s\colon X^s \to D^s$ the \emph{decision node}, \emph{utility node}, \emph{info node}, and \emph{infolink} of $s$, respectively, and refer to $\scontrol^s$, $\sinfo^s$ and $\sobspaths^s(C)$ for each collider $C$ in $\sinfo^s$ as the \emph{paths of $s$}.

\begin{definition}[Elements of a system] \label{def:back-&front-sections-of-system}
For a system $s$:
    \begin{itemize}
        \item An \emph{obs node} $O$ of $s$ is the penultimate node of each obs path $\sobspaths^s(C)$.\footnote{For ``observation node". But note that though $D^i$ does ``observe" $X^i$, it is not an obsnode, since it is not the penultimate node of an $\sobspaths^s(C)$, but is the first node of $\sinfo^s$.}

        \item The \emph{question node}  $Q^s$, if $\sinfo^s$ contains at least one fork node, is the closest-to-$U^s$ fork node on $\sinfo^s$.\footnote{This implies that the segment of the info path from $Q^s$ to $U^s$ is a directed path $Q^s\pathto U^s$, since there are no fork nodes on that path, and it must begin and end with an arrow towards $U^s$.}
        
        \item The \emph{back section}, if $\sinfo^s$ contains a fork, is the set of nodes in $X^s \upathto Q^s$ in $\sinfo^s$ 
        (including $X^s$ and $Q^s$)
        and in each $\sobspaths(C^i)$, except for $D^s$. Otherwise, the back section is empty.
        
        \item The \emph{front section} consists of the nodes in any path in $s$
        that are not in the back section.
    \end{itemize}
\end{definition}

\begin{definition}[Within-system links and paths] \label{sep12.4-within-system-within-tree-links}
A link $A\to B$ that is in $\sinfo^s$, $\scontrol^s$, or any $\sobspaths^s(C^i)$, or the link $X^s\to D^s$ for some system $s$ is called a \emph{within-system-$s$} or \emph{within-system link}. A \emph{within-system path} is a path that contains only within-system links.
\end{definition}

We will now prove a number of fundamental properties of systems.

\begin{lemma}[Basic properties of a system in a soluble ID graph] \label{le:20Nov7.1-Basic-properties-of-systems-with-sufficient-recall}
Any system $s$ in a soluble ID graph has the properties: 
\begin{enumerate}[label=(\alph*)]
    \item (No infolinks in the back-section) \label{le:20Nov7.1a-no-decisions-in-the-back-section} 
    The back section of $s$ can only contain a decision $D' \in \sD$ if $D'=X^s$, and the infopath $\sinfo^s$ is front-door. Moreover then $D'$ is not in any $\sobspaths^s(C^i)$.
    \chris{Above can be phrased more neatly. but not a priority.}
    
    \item (Infolinks in $s$ are descendants of $D^s$) \label{le:20Nov7.1b-info-path-decisions-observe-the-control-path}
    An information link $N \to D'$ for $D'\neq D^s$ can only be contained in a path in system $s$ if the control path $\scontrol^s$ contains a parent of $D'$, so that $D' \in \Desc(D^s)$.
    
    \item (Parents of ancestor decisions are parents of $D^s$) \label{le:20Nov7.1c-parents-of-ancestors-are-parents}
    A node $N$ in system $s$ can only be a parent of an ancestor decision $D'$ of $D^s$ if $N$ is also a parent of $D^s$.\footnote{Note that in a normal form tree (see below), this link $N \to D'$ is an out-of-tree link}
\end{enumerate}
\end{lemma}

\begin{proof} We prove each property in succession:

(a) (No infolinks in the back section) 
We will prove what restrictions are implies by considering sequentially the cases where 
$D'$ is in either the infopath, or in the observation path.
To begin with, let us state what we know in both cases:
$D'$ must be an ancestor of $D^s$. 
As such, $D'<D^s$ in any topological ordering, 
so solubility requires that $D' \dsep U^s \mid \Fa(D^s)$.

If $D' \in \sinfo^s$, then the path
$p\colon \Pi'\to D' \oset[0.5ex]{\sinfo^s}{\upathto} U^s$ may be
formed from by truncating the infopath $\sinfo^s$.
By solubility, $p$ must be blocked given $\Fa(D^s)$.
We know, however, that
$\sinfo\colon X^s \oset[0.5ex]{\sinfo^s}{\upathto} U^s$
is active given $\Fa({D^s}) \setminus \{X^s\}$.
If $D' \neq X^s$ then $p$ does not contain $X^s$, and so it is is active given $\Fa({D^s})$, 
violating solubility.
Moreover, if $D'=X^s$ and $\sinfo$ is a backdoor path, then $p$ will have a collider at $D'$, 
and solubility is violated once again.
So $\sinfo^s$ can only contain a decision $D'$ if $D'=X^s$ and $\sinfo^s$ is frontdoor.

Now we will prove that $D' \not \in \sobspaths(C)$, by contradiction.
Suppose that $D' \in \sobspaths(C)$.
Then, consider 
the path $q\colon \Pi'\to D' \oset[0.5ex]{\sobspaths^s(C)}{\dashleftarrow} C \oset[0.5ex]{\sinfo^s}{\upathto} U^s$,
constructed by truncating the observation path and infopath.
By assumption, the path $C\oset[0.5ex]{\sobspaths^s(C)}{\pathto} Y\to D$ is minimal-length, 
so no node $W\neq Y$ on the path can be a parent of $D^s$, 
and so $D^s \oset[0.5ex]{\sobspaths^s(C)}{\dashleftarrow} C$ is active given $\Fa(D^s)$.
The segment $C \oset[0.5ex]{\sinfo^s}{\upathto} U^s$ is active given $\Fa(D^s) \setminus \{X^s\}$.
Since $\sinfo^s$ is a path, the segment $C \oset[0.5ex]{\sinfo^s}{\upathto} U^s$ cannot contain $X^s$, 
and thus is active given $\Fa(D^s)$.
So the path $q$ is active given $\Fa(D^s)$, violating solubility.

Together, these two cases prove the result.

(b) (Infolinks in $s$ are descendants of $D^s$) 
We know from sublemma \ref{le:20Nov7.1a-no-decisions-in-the-back-section} that
the back section cannot contain any link $N \to D'$.
So $D'$ must lie in the front-section of $s$:
either in $Q^s \oset[0.5ex]{\sinfo^s}{\upathto}U^s$, or in $\scontrol^s$.
In either case, we have $U^s \in \Desc(D^s)$ and $U^s \in \Desc(D')$.
So in order for the ID graph to be soluble, 
we must have either $\Pi^{D'} \dsep U^s \mid \Fa(D^s)$ 
or $\Pi^{D^s} \dsep U^s \mid \Fa(D')$.

We can show that the first case $\Pi^{D'} \dsep U^s \mid \Fa(D^s)$ cannot hold.
If $D'$ is in $\scontrol^s$, note that $\scontrol^s$ consists of only descendants of $D^s$. 
If $D'$ is in $q:Q^s \oset[0.5ex]{\sinfo^s}{\upathto}U^s$ then note that $q$ is assumed to be active given $\Fa(D^s) \setminus \{X^s\}$, 
and cannot contain $X^s$.
In either case, $\Pi^{D'} \not \dsep U^s \mid \Fa(D^s)$.
Hence we must have $\Pi^{D^s} \dsep U^s \mid \Fa(D')$, from which
it follows that every directed path from $D^s$ to $U^s$ (including $\scontrol^s$) must contain a parent of $D'$.

(c) (Parents of ancestor decisions are parents of $D^s$) Assume $N$ is a parent of $D'$ in a path of $s$. 
It cannot be in $\scontrol^s$, because then $D'$ would be a descendant of $D^s$.
So $N$ must be in $\sinfo^s$ or one of $\sobspaths^s(C)$.
If $N$ is in $\sinfo^s$, consider the path
$p\colon \Pi^{D’} \to D’ \leftarrow N \oset[0.5ex]{\sinfo^s}{\upathto} U^s$.
We know $\sinfo^s$ is active given $\Fa({D^s}) \setminus \{X^s\}$.
Hence if $N \notin \Pa(D^s)$, then $p$ is active given $\Pa({D^s})$ and since it doesn't contain $D^s$ also active given $\Fa(D^s)$, violating solubility. Hence $N\in \Pa(D^s)$.

Similarly, if $N$ is in $\sobspaths^s(C)$, then consider the path
$q\colon \Pi^{D'}\to D'\leftarrow N\oset[0.5ex]{\sobspaths^s(C)}{\dashleftarrow} C \oset[0.5ex]{\sinfo^s}{\upathto} U^s$.
Hence if $N \notin \Pa(D^s)$, then since $\sobspaths^s(C)$ is minimal-length, it holds that $N\oset[0.5ex]{\sobspaths^s(C)}{\dashleftarrow} C$ is active given $\Fa(D^s)$, as in the proof of \ref{le:20Nov7.1a-no-decisions-in-the-back-section}.
Moreover, the segment
$C\oset[0.5ex]{\sinfo^s}{\upathto}U^s$ is active given $\Fa(D^s)\setminus \{X^s\}$ by assumption, and hence given $\Fa(D^s)$.
Since $D' \in \Anc(D^s)$, $q$ is active given $\Fa(D^s)$, again violating solubility. Hence again $N\in \Pa(D^s)$.
\end{proof}

\subsection{Trees of systems}
First, let us recap the definition of a tree of systems.

\deftree*
We define the \emph{predecessor system} 
and \emph{predecessor path} of system $s^i$ 
as $(\predsys{s^i},\predpath{s^i}):=\pred{s^i}$.
Moreover, we will sometimes say ``An ID graph with tree" to refer to an ID graph, together with a tree on that ID graph.

\begin{terminology*}
If $s^i=\predsys{s^j}$ then we say that $s^j$ is a child system of $s^i$. 
We will similarly apply the standard terminology of trees and graphs: Ancestor system, descendant system.
\end{terminology*}

\begin{lemma} [A tree of systems has a tree structure] \label{le:21Jan21.1-tree-of-systems-has-tree-structure}
Given a tree of systems $T=(\calS, \spred)$, the pair $(\calS,\spredsys)$ is a tree structure, i.e. it satisfies:
\begin{itemize}
    \item There is a unique node $\ssRoot$ that has $\spredsys(\ssRoot)=\ssRoot$; and
    
    \item For any node $s$, there is some number $n\in \mathbb{N}$ such that $\spredsys^n(s)=\ssRoot$.
\end{itemize}
\end{lemma}

\begin{proof}
The first condition is satisfied directly by definition of $\ssRoot$. For the second condition, we only need to show that for any $s^i$, there is a sequence of systems $(s^1,...,s^n)$ such that $s^1=\ssRoot$ and $s^n=s^i$, and $\predsys{s^j}=s^{j-1}$ for all $1<j\leq n$. Assume by contradiction that there is a system that doesn't satisfy this, and let $\calS^*$ be the set of all such systems. Then since the restrition of $\spredsys$ to $\calS^*$ has no fixed points ($\ssRoot$ is the only fixed point and is not in $\calS^*$ by definition), it must have some sequence $(\tilde s^1,...,\tilde s^k)$ with $\spredsys(\tilde s^1)=\tilde s^k$ and $\spredsys(\tilde s^j)=\tilde s^{j-1}$ for all $1<j\leq k$ (i.e. a cycle). But this would imply that there is at least one pair of systems $(\tilde s^k,\tilde s^{m})$ with $\spredsys(\tilde s^k)=\tilde s^m$ but where $D^{\tilde s^k}$ is a later decision than $D^{\tilde s^m}$, contradicting \autosubref{le:20Nov7.1-Basic-properties-of-systems-with-sufficient-recall}{le:20Nov7.1b-info-path-decisions-observe-the-control-path}.
\end{proof}

\newcommand{\decisionsindescendantsystemsaredescendantsstatement}{If $s'$ is a descendant system of $s$, then $D^{s'}$ is a descendant node of $D^s$.}

\begin{lemma}[Basic properties of a tree in a soluble ID]
\label{le:20dec14.1-basic-properties-of-trees-with-SR}
Let $T$ be a tree on a soluble ID graph. Then:
\begin{enumerate}[label=(\alph*)]
    \item (Decisions in descendant systems are descendants) \label{le:20dec14.1a-decisions-in-descendant-systems-are-descendants}
    \decisionsindescendantsystemsaredescendantsstatement

    \item (Info links to ancestor decisions only from obsnodes) \label{le:20dec14.1b-only-info-links-from-observation-nodes-to-ancestor-decisions}
    Let $s'$ be a descendant system of $s$. If there is a link from a node $V$ in 
    any path in
    $s'$ 
    to any node in $\sD \cap \Anc(D^s)$ (including $D^s$), then either: 
    i) $V$ is an obsnode in $s'$, 
    or ii) $V=X^{s'}=X^s$.
\end{enumerate}
\end{lemma}

\begin{proof} We prove each property in succession:

\sublemmaproof{a}{Decisions in descendant systems are descendants}
If $s'$ is a child system of $s$, 
then $D^{s'}$ is a descendant of $D^s$ by 
\autosubref{le:20Nov7.1-Basic-properties-of-systems-with-sufficient-recall}{le:20Nov7.1b-info-path-decisions-observe-the-control-path} 
(since it cannot lie in the back section by 
\autosubref{le:20Nov7.1-Basic-properties-of-systems-with-sufficient-recall}{le:20Nov7.1a-no-decisions-in-the-back-section}).
By induction the result follows: If any system $s'$ with child system $s''$ is a descendant system of $s$, then $D^{s''} \in \Desc(D^{s'})$, and by the induction assumption we know $D^{s'} \in \Desc(D^{s})$, so that $D^{s''} \in \Desc(D^{s})$.~\looseness=-1

\sublemmaproof{b}{Only info links from obsnodes to ancestor decisions}
Since $\sinfo^s$ is active, and each $\sobspaths^s(C)$ is a minimal length path,
the only parents of $D^{s'}$ within system $s'$ (i.e. the only nodes in $\Pa(D^{s'})\cap \sV^{s'}$) are $X^{s'}$ and the obsnodes of $s'$. 
Therefore, by \autosubref{le:20Nov7.1-Basic-properties-of-systems-with-sufficient-recall}{le:20Nov7.1c-parents-of-ancestors-are-parents} and using Sublemma \ref{le:20dec14.1a-decisions-in-descendant-systems-are-descendants} that $D^s$ is an ancestor of $D^{s'}$ (since $s$ is an ancestor of $s'$), these are the only nodes in $s’$ that can be parents of $D^s$ or of ancestor decisions of $D^s$. 

To show the result we show that $X^{s'}$ cannot be such a parent when {$X^{s'} \neq X^s$}:
Let $s^*$ be the closest-to-$s'$ ancestor of $s'$ in the tree such that {$X^{s*}\neq X^{s'}$}. Assume such $s^*$ exists and either equals $s$ or is a descendant of $s$ since otherwise $X^{s}$ would equal $X^{s'}$, which would show the result.
We know that $X^{s'}$  is in the system $s^*$, since it is the closest-to-$s'$ system such that {$X^{s*}\neq X^{s'}$}, so that there is a child system of $s^*$ whose info node equals $X^{s'}$ {} and hence must be part of an info-link in $s^*$. 
Hence $X^{s'}$ cannot be a parent of $D^{s^*}$ since the only parents within a system of that system's decision other than its info node are its obsnodes, but $X^{s'}$  cannot be one of the obsnodes since then $D^{s'}$ would have to be in the back section, which would violate \autosubref{le:20Nov7.1-Basic-properties-of-systems-with-sufficient-recall}{le:20Nov7.1a-no-decisions-in-the-back-section}{}.
But we assumed $s^*$ is a descendant system of $s$, and hence $D^{s^*}$ is a descendant decision of $D^s$ (by Sublemma \ref{le:20dec14.1a-decisions-in-descendant-systems-are-descendants}), which implies that $D^s$ and its ancestor decisions also don't have $X^{s'}$ as a parent (due to the result shown in the previous paragraph).~\looseness=-1
\end{proof}

\subsection{Normal form trees of systems} \label{sec:normal-form-trees}

In this section, we will prove that 
in a \emph{normal form tree},
a system can only get information from 
its own parents, and obsnodes of descendant systems.

\defnormalform*

We will also use the components of the definition of normal form tree separately.

\begin{lemma}[Concrete properties of position-in-tree-uniqueness]\label{le:21-may-9-position-in-tree-uniqueness-properties}
A tree $T$ satisfied position-in-tree-uniqueness if and only if every node $N$ that is in some path of some system in $T$ lies in precisely one path $p$ of one system $s$, with four exceptions:
    \begin{itemize}
        \item If $N$ is a collider node in 
        path $p=\sinfo^s$ 
        then it is also the first node in 
        $\sobspaths^s(N)$.
        
        \item If $N=U^s$ then it lies in both $\sinfo^s$ and $\scontrol^s$.
        
        \item If $N$ is in an infolink $X^{s'} \to D^{s'}$ (with $s' \neq s$) on path $p$, 
        then $N$ is also in $\sinfo^{s'}$ (if $N=X^{s'}$), or also in $\scontrol^{s'}$ and in $\sobspaths^{s'}(C^i)$ for each collider $C^i$ in $\sinfo^{s'}$ (if $N=D^{s'}$). In both cases 
        $N$ may also be the info node for exactly one of its child systems $s^1$, of exactly one child system $s^2$ of $s^1$, and so on.
        Formally, $N=X^{s^1}=...=X^{s^n}$ where each $s^{i}$ is a child system of $s^{i-1}$.
    \end{itemize}
\end{lemma}
\begin{proof}
First we show that if a tree $T$ satisfies position-in-tree-uniqueness, then the result is true. Assume that a tree $T$ does not satisfy the required property, i.e. there is at least one node $N$ that is in multiple paths, but without satisfying one of the exceptions. Then by \cref{le:21jan25.6-first-split-preserves-tree} a different tree $T'$ can be obtained by applying graph transformation 1 (\cref{def:21jan25.2-first-split-to-ensure-systems-and-paths-uniqueness}), where $N$ is replaced with different nodes in those paths.~\looseness=-1

Now we show the other direction. Assume that the property holds.
Assume $N$ is part of two paths $p^1$ and $p^2$. Then one of the three exceptions must apply. If the first exception applies, then $N$ is a collider, and $p^1=\sinfo^s$ and $p^2=\sobspaths^s(N)$, so that replacing $N$ with two separate nodes on $p^1$ and $p^2$ would make that the obspath of $N$ no longer starts with a collider on $\sinfo^s$. If the second one applies, then replacing $N=U^s$ with two nodes would mean that the control and info path no longer end at the same utility node. If the third case applies and $N=X^{s'}$ of a  descendant system $s'$ of $s$, then $p^1, p^2$ equal $\sinfo^{s'}$ and $\spredpath^{s'}$. Replacing $N$ with two nodes on the two paths would break the required property on $\spred$ for $T$ to be a tree. If the third case applies and $N=D^{s'}$, then $p^1,p^2$ equal two of: $\scontrol^{s'}$, $\spredpath^{s'}$ or one of $\sobspaths^{s'}$. If one of them equals $\spredpath^{s'}$, then replacing $N$ with two nodes would again break the required property on $\spred$ for $T$ to be a tree. Otherwise, it would mean that at least one of the $\sobspaths^{s'}$ no longer ends at $D^{s'}$, so that $s'$ would no longer be a system.~\looseness=-1
\end{proof}

\begin{definition} [Base system and path of a node; chain of systems] \label{def:21feb17.2-base-sytem-and-path-of-node}
If $T$ is a normal form tree of systems, then we refer to the system $s$ and the path $p$ from
\cref{le:21-may-9-position-in-tree-uniqueness-properties}
(including in the exceptions) respectively as the \emph{base system} and \emph{base path} of node $N$. 
\end{definition}
Note that this implies that a utility node $U^s$ has no base path. We refer to the sequence of systems of which a node $N$ is the info node (in the third exception) as the \emph{chain of systems} of $N$ (which is possibly empty).

\begin{definition}
A within-tree-$T$ path for a normal form tree $T$ on an ID graph is a path that contains only within-system links for the systems in $T$.
\end{definition}

Note that we define the notion of within-tree path only for normal form trees, since it is not sensible for trees that don't satisfy \systemsAndPathsUniqueness: If a node $N$ occurs in two unrelated systems, then a sequence of within-tree links may jump between nodes in the tree that are not linked.

\subsection{Properties of normal form trees of systems} 
\label{sec:graph-knowledge-lemma-proof}

In this subsection, we will prove \cref{le:2v2-graph-knowledge-lemma} --- that 
the only information that $D^s$ receives that is relevant within system $s$ is information that it receives from its parents and obsnodes of descendant systems.
To reach this result, we first need to state some more fundamental properties of normal form trees.

\newcommand{\atmostonebackdoorsystempernodestatement}{For any node there is at most one backdoor-info system $s$ of which it is one of the nodes in the link $X^s\to D^s$. }
\newcommand{\psys}{p_\mathrm{systems}}
\newcommand{\pwalk}{p_\mathrm{walk}}
\begin{lemma}[Properties of soluble ID graphs with trees that have \systemsAndPathsUniqueness]
\label{le:20Nov24.1-basic-properties-of-trees-with-sufficient-recall-and-unique-systems-and-paths}
Any soluble ID graph $\sG$ with a tree that has \systemsAndPathsUniqueness
has the following properties.
\begin{enumerate}[label=(\alph*)]

    \item (A within-tree path corresponds to a walk with node repetition in the tree of systems) \label{le:20Nov24.1a1-Within-tree-paths-correspond-to-paths-in-the-tree-of-systems}
    For any within-tree path $p\colon  N^1 \upathto N^n$, there is a walk with node repetition in the tree of systems $\psys\colon  s^1 \upathto s^m$, with $m\geq n$, together with a walk with node repetitions $\pwalk\colon V^1 \upathto V^m$ in $\sG$ such that each $V^i$ is in some path in system $s^i$ and if we remove from $\pwalk$ every node that equals its predecessor we obtain $p$.\footnote{Hence in particular, there can only be a within-tree path between a node $N^1$ in system $s$ and node $N^n$ in system $s'$ if there is a path between $s$ and $s’$ in the tree of systems.}

    \item (Within-tree links between systems only via $X^s$, $D^s$) \label{le:20Nov24.1a2-Within-tree-links-between-systems-only-via-Xs-Ds}
    If $N - N'$ is a within-tree link, where $N$ and $N'$ are in nodes in paths of systems $s$ and $s'$ respectively, and $s \neq s'$,
    then $N - N'$ must contain $X^s$ or $D^s$.
    
\end{enumerate}
\end{lemma}

\begin{proof} We prove each sublemma in succession:

\sublemmaproof {a} {Within-tree paths correspond to walks with node repetition in the tree of systems}.
We construct this walk with node repetition $\psys$ recursively as follows, by iterating from $N^1$ to $N^n$, using the fact that each link in $p$ is within-system for some system (see definition of within-tree paths). For the base case, let $s^1$ equal any of the systems that $N^1$ is a node in. Let $s^{k+1}$ and $V^{k+1}$ be defined mutually based on $s^{k}$ and $V^{k}$: If the node $N’$ that is next to $V^k$ on $p$ is also in system $s^k$, then let $s^{k+1}=s^k$ and let $V^{k+1}=N’$, in which case the desired result follows that $s^k=s^{k+1}$ and that $V^{k} - V^{k+1}$ is a link in $p$. If it is not also in system $s^k$, then by definition of within-tree path, $N’$ and $V^k$ are both in some system $s’ \neq s^k$, where $s’$ is part of the chain of systems of $V^k$. Then let $s^{k+1}$ be the next system from $s^k$ in that chain, and let $V^{k+1}=V^k$, from which the desired result follows that there is a link $s^k-s^{k+1}$ and $V^{k+1}=V^k$. Together with the base case this shows the result by induction.~\looseness=-1

\sublemmaproof{b} {Within-tree links between systems only via $X^s$, $D^s$}. 
Take any link $A - B$ with $A$ a part of $s$ and $B$ a part of some other system $s’$.
Then we must either have that $A$ is in both $s$ and in $\predsys{s}$, or that $B$ is in both $s$ and $\predsys{s}$. Whichever it is, by the \systemsAndPathsUniqueness assumption, this can only be if that node equals $X^s$ or $D^s$, since any node that is in multiple systems $s’$ must equal either $X^{s'}$ or $D^{s'}$ for all systems $s'$ except its base system.
\end{proof}

We now show graphically that in an ID graph with normal form tree a decision $D^s$ cannot get relevant information about system $s$ from any paths via nodes outside system $s$
and descendant systems. 
This will imply the following:

\graphKnowledgeLemma*
\ryan{graph knowledge lemma numbered as 43 here, but as 20 in the main paper}

\newcommand{\OutsideParents}{\Pa^{-s} \setminus \ObsDesc^s}
\newcommand{\InsideParents}{\Pa^{s} \cup \ObsDesc^s}
\Chris{I should try to simplify the proof below further.}
\begin{proof} Take any path from a node in $\Back^s$ to a node in $\OutsideParents$. We will show that it is inactive given $\InsideParents$.

We first assume that the path starts from a node in the back section, so that it is an ancestor of $D^s$.
First note that since the decision of $s$ and those of its descendant systems cannot be ancestors of $D^s$ (\autosubref{le:20dec14.1-basic-properties-of-trees-with-SR}{le:20dec14.1a-decisions-in-descendant-systems-are-descendants}), this implies that if the path contains any of these it is necessarily inactive given {$\InsideParents$}, since active paths between ancestors given ancestors contain only ancestors (\cref{le:2.12-active-paths-between-ancestors-contain-only-ancestors}).

So assume that the path does not contain the decision of $s$ (i.e. $D^s$) nor those of its descendant systems. We will consider the initial within-tree segment of the path. 

By \autosubref{le:20Nov24.1-basic-properties-of-trees-with-sufficient-recall-and-unique-systems-and-paths}{le:20Nov24.1a1-Within-tree-paths-correspond-to-paths-in-the-tree-of-systems}, this initial within-tree path corresponds to a walk with node repetition in the tree of systems, and since the latter has a tree structure by construction, this initial within-tree path either has to exit system $s$ via a node in its predecessor system, or stay within $s$ itself and its descendant systems. 
The former can only happen via one of the links via $X^s$ and $D^s$ by \autosubref{le:20Nov24.1-basic-properties-of-trees-with-sufficient-recall-and-unique-systems-and-paths}{le:20Nov24.1a2-Within-tree-links-between-systems-only-via-Xs-Ds}.

We first show that in this case, the path is blocked given $\InsideParents$. We already assumed that the path doesn’t contain $D^s$, so assume that the link contains $X^s$. Since $\sinfo^s$ is front-door by the appropriateness assumption of normal form tree, $X^s$ blocks the path, since $X^s\in \InsideParents$.

So we now assume that the initial within-tree segment does not exit into the predecessor of $s$, and hence is contained within system $s$ and its descendant systems. Consider the first link of the path that is out-of-tree. At the start of this proof we assumed that the path doesn't contain the decisions $D^s$, nor $D^{s'}$ of any of its ancestors $s'$. Hence by \autosubref{le:20Nov7.1-Basic-properties-of-systems-with-sufficient-recall}{le:20Nov7.1a-no-decisions-in-the-back-section}{}, the only decision that the initial within-tree segment could contain is $X^s$ if that is a decision, but we just assumed that the path doesn't contain this.~\looseness=-1

So we assume now that the initial within-tree segment doesn't contain any decisions, so that the first out-of-tree link would have to be of the form $N \to D$ for some decision $D$ (by the no redundant links assumption of normal form trees, and using the fact that $N$ is inside the tree). $N$ can be either in system $s$ or in a descendant system, and can be either an observation node or some other node. Consider two exhaustive cases:

\begin{enumerate}[label=(\alph*)]
    \item Assume $N$ is neither an obsnode in $s$ nor in a descendant system $s’$. Then $D$ cannot be an ancestor of $D^s$, since the only info links from nodes in $s$ or its descendant systems to $D^s$ or to ancestor decisions are from obsnodes by \autosubref{le:20dec14.1-basic-properties-of-trees-with-SR}{le:20dec14.1b-only-info-links-from-observation-nodes-to-ancestor-decisions}, and hence the path cannot be active (active paths between ancestors given ancestors contain only ancestors by \cref{le:2.12-active-paths-between-ancestors-contain-only-ancestors}).
    
    \item Assume $N$ is an observation node of $s$  or of some descendant system of $s$. Then $N$ blocks the path, since it is in $\Pa^s\cup \ObsDesc^s$.
    
\end{enumerate}
This shows the result.
\end{proof}

\subsection{Obtaining a (homomorphically) transformed ID graph with a normal form tree} \label{sec:cidhom-graph-splitting} 

We will prove that
if an infolink is in the minimal $d$-reduction, then
there exists a transformed ID graph with a normal form tree and homomorphism to the original.
We will show that a series of three homomorphic transformations %
yields a graph $\sG^3$ with tree $T^3$ is in normal form, and
root infolink corresponding to that of $T$.
Since each transformation is homomorphic, 
their composition is a
homomorphism from $\calG^3$ to $\calG$.
The transformations are:~\looseness=-1

\begin{itemize}
    \item First, we \emph{obtain a full tree} on $\calG$.
    \item \emph{Transformation 1} obtains
    $(\sG^1,T^1)$, where $T^1$ has property (a).
    This splits nodes other than $X^s$ and $D^s$, to ensure that 
    they do not appear in multiple positions in the tree.~\looseness=-1
    
    \item \emph{Transformation 2} obtains 
    $(\sG^2,T^2)$, where $T^2$ has the properties (a, b).
    This is done by modifying any backdoor infopath to be front-door.
    
    \item \emph{Transformation 3} obtains 
    $(\sG^3,T^3)$, where $T^3$ has the properties (a, b, c).
    This consists of removing edges other than the within-tree links.
\end{itemize}
We will not use the intermediate graphs $\sG^1,\sG^2$, except to define $\sG^3$.

\subsubsection{Obtain a full tree on $\calG$}
We will construct an arbitrary full tree using only infolinks in the minimal $d$-reduction.

\begin{lemma}[Existence of full tree]
\label{le:m3.1-existence-of-complete-tree-of-systems}
Let $\calG$ be a soluble ID graph whose minimal $d$-reduction $\calG^*$ contains the link $X \to D$.
Then there exists a full tree of systems $T$ on $\sG^*$ with root system on $X \to D$.
\end{lemma}

\begin{proof}
We construct a tree iteratively. 
Since $X \to D$ is in $\sG^*$, 
there exists a directed path $p$ from $X$ to some $U \in \Desc(D)$ active given $\Fa(D^i) \setminus \{X\}$.
Let the infopath be any such $p$, let the control path be any directed path from $D$ to $U$ and 
let the obspaths be the shortest directed paths to $D$ from each collider in $p$.
Then, choose any infolink $X' \to D'$ 
in a path $q$ of any system $s$ 
that lacks an associated system.
Since $X' \to D'$ is in $\calG^*$, 
we can choose paths in the same fashion
and repeat this procedure 
until every infolink that is traversed 
has its own system.
This process halts, because a
path in a system $s$ (whose decision is $D^s$)
can only contain an infolink $X' \to D'$ if $D' \in \Desc(D^s)$.
This is because:
i) $\scontrol^s$ is directed,
ii) $\sinfo^s$ only contains infolinks in $\Desc(D^s)$
(\autosubref{le:20Nov7.1-Basic-properties-of-systems-with-sufficient-recall}{le:20Nov7.1b-info-path-decisions-observe-the-control-path}),
and iii) $\sobspaths^s$ cannot contain any infolinks.
So a full tree has been constructed.~\looseness=-1
\end{proof}

\subsubsection{Transformation 1 (split): ensuring \systemsAndPathsUniqueness} \label{subsubsec:split-1}

For Transformation 1, a node $N$ is copied into a different node
(of unchanged type)
for each position that $N$ occupies in the tree. More precisely, we replace each node $N$ that is in path $p$ in system $s$, with the new node $\NewNode(N,s,p)$. This function is defined such that each node $\NewNode(N,s,p)$ has a unique position in the tree, which basically means that it is only a part of path $p$ in system $s$, except that we need to make sure that certain nodes are in multiple paths (e.g. a collider node $C$ in a path $\sinfo^s$ must be in both $\sinfo^s$ and in $\sobspaths^s(C)$). We don't delete the original occurrences of each node $N$, so that the original graph is a subgraph of the transformed graph.

\ryan{Seems we should be able to remove the name $Nsplit$ and just 
substitute in New()}

\begin{definition}[Graph transformation 1]
\label{def:21jan25.2-first-split-to-ensure-systems-and-paths-uniqueness}
Let $T^0=(\calS^0,\spred^0)$ be a tree on an ID graph $\sG^0=(\sV^0,E^0)$. Then define $\splitfirst(\sG^0,T^0)=(\sG^1, T^1)$
where $\sG^1=(\sV^1,E^1)$, together with homomorphism $\hfirst\colon\sV^1 \to \sV^0$ as

\begin{itemize}
    \item Obtain any $\sG^1$ and $\hfirst$ from \autoref{le:21may19.2-CID-hom-from-node-copying-and-deleting}, by adding for each node $N$ a set of copies
    \[\sCopies(N) = \{\sNsplit \mid \textnormal{$\exists s\in \calS^0, \exists$ path $p \in s$ such that $N\in p$, and  $\sNsplit=\NewNode(N,s,p)$ } \},\]
    where by tree recursion on $T^0$ (which has tree structure: \autoref{le:21Jan21.1-tree-of-systems-has-tree-structure}) we define $\NewNode(N,s,p)$ \[= \begin{cases}
        \NewNode(N, \predsys{s}, \predpath{s})	\casesif {$s\neq \rootsys{T^0}$ and $N\in \{ X^s, D^s \}$}\\
        (N, s, \{\sinfo^s, \scontrol^s\}) 		\casesif {$N = U^s$}\\
        (N, s, \sinfo^s)	\casesif {$p = \sobspaths^s(N)$}\\ 
        N \casesif {$s=\rootsys{T^0}$ and $N\in\{X^s,D^s$\}}\\
        (N, s, p)			 	\casesotherwise
        \end{cases}.
    \]

    \item $T^1$ is the tree $(\calS^1,\spred^1)$, where 
    the system $\sssplit^i\in \calS^1$ is defined as the system $(\ssplit(s^i,\sinfo^{T,s^i}), \ssplit(s,\scontrol^{T,s^i}), \ssplit(s,\sobspaths^{T,s}))$ for $s^i\in \calS^T$, where 
    \(\ssplit(s^i,p)^j = \NewNode(p^j,s^i,p),\)
    and where $p^j$ denotes the $j$'th node of a path $p$. (this indeed gives a path, since there is an edge between $\ssplit(s^i,p)^j=\NewNode(p^j,s^i,p)$ and $\ssplit(s^i,p)^{j+1}=\NewNode(p^{j+1},s^i,p)$ because there is an edge between $p^j$ and $p^{j+1}$ and by definition of $E^1$ using $p^j\neq p^{j+1}$). Moreover $\spred^1$ is the same as $\spred^0$ except that each $s$ in $T^0$ is replaced with its transformed $\sssplit$.~\looseness=-1
\end{itemize}
\end{definition}

\begin{lemma}[Transformation 1 preserves tree] \label{le:21jan25.6-first-split-preserves-tree}
Let $(\sG^1, T^1)=\splitfirst(\sG^0,T^0)$. If $T^0$ is a tree of systems  on $\sG^0$ with root link $X\to D$, then $T^1$ is a tree of systems on $\sG^1$ with root link $X'\to D'$ with $\hfirst(X')=X$ and $\hfirst(D')=D$.
\end{lemma}

\begin{proof}
First we show that $T^1$ satisfies the three conditions of a tree of systems: 
(1) We will show below that each indexed element of $\calS^1$ is indeed a system.
(2) since $\spred^0 = \spred^1$, and since $T^0$ is a tree of systems, the required condition on $\spred^1$ is satisfied. (3) We show that each system's infolink is an infolink on its predecessor path: The nodes $X^{\sssplit}$ and $D^{\sssplit}$ in $\sG^1$ equal $\NewNode(X^s,s,\sinfo^s)$ and $\NewNode(D^s, s, \scontrol^s)$ for $X^s$ and $D^s$ in $\sG^0$. By definition of $\NewNode(N,s,p)$, this is indeed an infolink on $\predpath{s}$. 

It remains to be shown that $\sssplit$ indeed is a system for each $s \in \calS^0$: 

    \proofstepinf{1} {We show that $\scontrol^{\sssplit}$ is a directed path to a utility node}. $\scontrol^{\sssplit}$ is a path from  $\NewNode(D^{s}, {s}, \scontrol^{s})$ to $\NewNode(U^{s}, {s}, \scontrol^{s})$ and by definition of $E^1$ this is indeed a directed path (since $\scontrol^{s}$ is directed in $T$); 
    
    \proofstepinf{2} {We show that $\sinfo^{\sssplit}$ is an active path to the same utility node}. Firstly, $\sinfo^{\sssplit}$ is a path from $\NewNode(X^{s}, {s}, \sinfo^{s})$ to $\NewNode(U^{s}, {s}, \sinfo^{s})$, and since $\NewNode(U^{s}, {s}, \sinfo^{s}) = \NewNode(U^{s}, {s}, \scontrol^{s})$ by definition of $\NewNode$ for utility nodes, therefore the control and info path indeed end at the same utility node. Secondly, it follows easily from the definition of $E^1$, that a node $\NewNode(N,s,\sinfo^{s})$ blocks the path if and only if $N$ blocks $\sinfo^{s}$, and the latter is active by assumption, so that $\sinfo^{\sssplit}$ is active as well;

    \proofstepinf{3}{Finally, we show that the $\sobspaths^{\sssplit}$ are minimal length paths from collider nodes on $\sinfo^{\sssplit}$ to $D^{\sssplit}$}. 
    Firstly, $\sobspaths^{\sssplit}(\NewNode(C,s,\sinfo^{s}))$ is a path from $\NewNode(C, s,\sobspaths^{s}(C))$ to $\NewNode(D^s,s,\sobspaths^{s}(C))$. By definition of $L$, the former equals $(C,(s,\sinfo^{s}))$ and the latter equals $\NewNode(D^s,\predsys{s},\predpath{s})$ if $s\neq \rootsys{T}$ and $(D^s,(s,\scontrol^s))$ if $s=\rootsys{T}$, which in both cases equals $\NewNode(D^s, s,\scontrol^s)=D^{\sssplit}$, so that this is indeed a valid obspath. To show that it's minimal length, assume by contradiction that there is a shorter path and denote its $j$'th node by $(N^j,(s^j,p^j))$, so that there are links $(N^j,s^j,p^j)\to (N^{j+1},s^{j+1},p^{j+1})$. Then by definition of $E^1$, $\sG^0$ contains an edge $N^j\to N^{j+1}$, and hence this path in $\sG^0$ must also be shorter than $\sobspaths^s(C)$, contradicting the assumption that $s$ is a system.
    
This shows that $T^1$ is a tree of systems. Finally, The root infolink of $T^1$ is $X\to D$ by definition of $\NewNode$, and $\hfirst(X)=X$ and $\hfirst(D)=D$, so it is mapped to the root info link of $T$.
\end{proof}

\begin{lemma}[Transformation 1 ensures \systemsAndPathsUniqueness]
\label{le:21jan26.1-first-split-ensures-systems-and-paths-uniqueness}
Let $(\sG^0,T^0)$ be any soluble ID graph with complete tree. Then $(\sG^1,T^1)=\splitfirst(\sG^0,T^0)$ is an ID graph with complete tree that satisfies (a) \systemsAndPathsUniqueness.
\end{lemma}

\begin{proof} We first show that the split preserves fullness, then that it ensures \systemsAndPathsUniqueness.

    (i) {(full tree)}. 
    \chris{my sense is this subproof can be simplified, it seems very long for what it's doing. Basically it's a "by definition" kind of proof.}
    We show that if ${T^0}$ is a full tree, then so is $T^1$: 
    Let $(X,s^{*,1},p^{*,1}) \to (D,s^{*,2},p^{*,2})$ be an infolink in $\sG^1$ on the path $\psplit$ in system $\sssplit$. We need to show that there is a system $\sssplit'$ such that $X^{\sssplit'}=(X,s^{*,1},p^{*,1})$ and $D^{\sssplit'}=(D,s^{*,2},p^{*,2})$. 
    
    Note that this link in $\psplit$ implies that there is a corresponding link in the original path $p$. By the definition of $T^1$, the split path $\psplit$ was constructed from the original path $p$ (that has the same path type in system $s$ as $\psplit$ does in $\sssplit$) where if $(X,s^{*,1},p^{*,1})$ is the $i$'th node on $\psplit$, it corresponds to the $i$'th node on $p$ by $(X,s^{*,1},p^{*,1})=\psplit^i=\NewNode(p^i,s,p)$ and similarly $(D,s^{*,2},p^{*,2})=\psplit^{i+1}=\NewNode(p^{i+1},s,p)$. Hence $X=p^i$ and $D=p^{i+1}$, so that $X\to D$ is also an infolink on $p$ in $s$ in ${T^0}$.
    And since by assumption ${T^0}$ is full, there is a system $s'$ with $\spred^0(s')= (s,p)$ such that $X^{s'} \to D^{s'}$ equals $X \to D$. This implies also that the desired system in $T^1$ exists: Since by definition $\spred^1$ is equivalent to $\spred^0$ it implies that $\spred^1(\sssplit')= (\sssplit,\psplit)$, and by construction of $\sssplit$,   $X^{\sssplit'}$ and $D^{\sssplit'}$ equal $\NewNode(X^{s'},s',\sinfo^{s'})=\NewNode(X,s',\sinfo^{s'})$ and $\NewNode(D^{s'},s',\scontrol^{s'})=\NewNode(D,s',\scontrol^{s'})$ respectively, which by definition of $\NewNode(X,s,p)$ implies that they equal $\NewNode(X,s,p)$ and $\NewNode(D,s,p)$ respectively, showing the result. 
    
    {(ii)} {(\systemsAndPathsUniqueness)}.
    \newcommand{\uspindprop}{\mathcal P}
    Any node $\sNsplit$ in the tree either equals one of $X,D$, or is a node of the form $\sNsplit=(N,(s^*,p^*))$. In the former case, let $s^*=\rootsys{T^1}$ and let $p^*=\sinfo^{\sRoot}$ if the node equals $X$ and $p^*=\scontrol^{\sRoot}$ if it equals $D$. We will show that $\sssplit^*$ and $\psplit^*$ are the node's base system and base path respectively, by taking any path $\psplit$ in any system $\sssplit$ such that the node is on $\psplit$, and showing that for the original path $p$ and system $s$, either $p=p^*$ and $s=s^*$ or that one of the exceptions applies.
    
    We will show this by induction on the tree: Assume that it holds for any $\psplit'$ in system $\sssplit'$ that is an ancestor system of $\sssplit$. Note that if $\sNsplit=(N,(s^*,p^*))$ is on path $\psplit$ in system $\sssplit$, then $\sNsplit=\NewNode(N,s,p)$, so consider two cases of the definition of $\NewNode(N,s,p)$ separately:
    
    \proofcaseinf{1} {Assume $s\neq \rootsys{T^0}$ and $N\in \{ X^s, D^s \}$.} Then $\sNsplit = (N, s^*, p^*) = \NewNode(N, s, p) = \NewNode(N, \predsys{s}, \predpath{s})$, and we will use the induction assumption on $\predsys{\sssplit}, \predpath{\sssplit}$: We know that $(N, s^*, p^*)$ lies on $\predpath{\sssplit}$ (since $T^1$ is a tree), and by the induction assumption, either $\predsys{\sssplit}=\sssplit^*$ and $\predpath{\sssplit}= \psplit^*$ (in which case the third or fourth exception applies to $\sssplit$ and $\psplit$, showing the result), or one of the exceptions applies. Since $X^s$ and $D^s$ aren't utility nodes, and can't be colliders on the info path of $\predsys{s}$ (\autosubref{le:20Nov7.1-Basic-properties-of-systems-with-sufficient-recall}{le:20Nov7.1a-no-decisions-in-the-back-section}), and $\spredpath$ always is either an info or control path, only the third exception can apply to $\predsys{\sssplit}$ and $\predpath{\sssplit}$, i.e. $(N, s^*, p^*)$ is the info or decision node of $\predsys{\sssplit}$, where the latter is a child system of $\sssplit^*$ or it is the info node of an unbroken chain of descendant systems of $\sssplit^*$. In both cases, $N$ cannot equal $D^s$, since the decision node of a system ($\predsys{\sssplit}$ in this case) is neither the decision of an info link on its info path nor on its control path and hence cannot equal the decision node of one of its child systems ($\sssplit$ in this case), and hence $(N, s^*, p^*)$ must be the info node of $\sssplit$ and an unbroken chain of predecessor systems between $\sssplit$ and $\sssplit^*$, so that it satisfies the third exception.

    \proofcaseinf{2} {Assume $s= \rootsys{T^0}$ or $X \notin \{X^s, D^s\}$.} Then if $s=\rootsys{T^0}$ and $N\in \{X^s, D^s\}$, then $\sNsplit = N$, and the result follows easily (where $\psplit$ may be an obs path in which case the final exception applies). So assume otherwise, so that  $\sNsplit = (N, s^*, p^*)= \NewNode(N, s, p)$, where $s^*=s$, and we can easily match each of the cases of $\NewNode(N, s, p)$ to the exceptions, showing the result.~\looseness=-1
\end{proof}

\subsubsection{Transformation 2 (split): ensuring no backdoor info-paths}
In the second transformation, we turn any backdoor-info paths into frontdoor infopaths.

\begin{definition}[Transformation 2]
\label{def:21jan28.1-third-split-to-ensure-new-appropriateness}
Let $\sG^1=(\sV^1,E^1)$ be an ID graph with tree $T^1=(\calS^1,\spred^1)$. Then $\splitsecond(\sG^1,T^1)=(\sG^2, T^2)$, where $\sG^2=(\sV^2,E^2)$ and $\hsecond\colon\sV^2 \to \sV^1$ are defined as follows:~\looseness=-1
\begin{itemize}
    \item Obtain any $\sG^2$ and $\hsecond$ from \autoref{le:21may19.2-CID-hom-from-node-copying-and-deleting}, by adding for each node $N$ a set of copies
    \[\sCopies(N) = \begin{cases} 
        \{(N,``\scopy", s)\} \casesif {$N=X^{s}$ for some backdoor-info system $s$ in ${T^1}$}\\
        \emptyset \casesotherwise
    \end{cases}\]

    \item \sloppy $T^2$ is the tree $(\calS^2,\spred^2)$, where each system $\sssplit\in \calS^2$ is obtained from $s$ by replacing the first link $X^s\gets N$ in $\sinfo^s$ with the links $X^s\to (X^s,``\scopy")\gets N$, and extending $\sobspaths^s$ with $\sobspaths^s((X^s,``\scopy"))$ to be the path consisting of the single link $(X^s,``\scopy", s)\to D^s$. Moreover $\spred^2$ is the same as $\spred^1$ except that each $s$ in $T^1$ is replaced with its transformed $\sssplit$. 
\end{itemize}
\end{definition}

\begin{lemma}[Transformation 2 preserves tree]\label{le:21jan28.3-third-split-preserves-tree}
Let $(\sG^2, T^2)=\splitsecond(\sG^1,T^1)$. If $T^1$ is a tree of systems on $\sG^1$ with root link $X\to D$, then $T^2$ is a tree of systems on $\sG^2$ with root link $X'\to D'$ with $\hsecond(X')=X$ and $\hsecond(D')=D$.
\end{lemma}

\begin{proof}
First we show that $T^2$ satisfies the three conditions of a tree of systems: (1)
We will show below that each indexed element of $\calS^2$ is indeed a system.
(2) since $\spred^1$ is equivalent to $\spred^2$, and since ${T^1}$ is a tree of systems, the required conditions on $\spred^2$ are satisfied.
(3) For any system $s$, $X^{\sssplit}=X^s$ and $D^{\sssplit}=D^s$ (i.e. they are unchanged under the split), and the front-section of $\predsys{\sssplit}$ is identical to that of $\predsys{s}$, and since infolinks can only be in the front section (by \autosubref{le:20Nov7.1-Basic-properties-of-systems-with-sufficient-recall}{le:20Nov7.1a-no-decisions-in-the-back-section}), so that the fact that ${T^1}$ is a tree of systems and hence has $X^s\to D^s$ as an infolink on $\predpath{s}$, this implies that $X^{\sssplit}\to D^{\sssplit}$ is an infolink on $\predpath{\sssplit}$.

It remains to be shown that $\sssplit \in \calS^2$ indeed is a system for each $s \in \calS^1$:
    
    \proofstepinf{1} {$\scontrol^{\sssplit}$ is a directed path to a utility node}. 
    $\scontrol^{\sssplit}$ is identical to $\scontrol^s$.
    
    \proofstepinf{2} {$\sinfo^{\sssplit}$ is an active path to the same utility node}.
    Since $\sinfo^s$ is active given $\Pa(D^s)$ by assumption, and $\sinfo^{\sssplit}$ is identical to $\sinfo^s$ except that if $\sinfo^s$ is backdoor from $X^s$ then the first link $X^s\gets N$ is replaced by $X^s\to (X^s, ``\scopy")\gets N$, hence by definition of $E^2$, a node on $\sinfo^{\sssplit}$ is a parent of $D^{\sssplit}$ in $\sG^2$ iff it is a parent of $D^s$ in $\sG^1$ or if it equals $(X^s, ``\scopy")$. Therefore, $(X^s, ``\scopy")$ doesn't block because it's a collider, and the other nodes don't block $\sinfo^{\sssplit}$ because by assumption they didn't block $\sinfo^s$.
    
    \proofstepinf{3}{Finally, $\sobspaths^{\sssplit}$ are minimal length paths from collider nodes on $\sinfo^{\sssplit}$ to $D^{\sssplit}$}: For $\sobspaths^\sssplit((X^s,``\scopy"))$, it is a single link to $D^s$ and hence trivially minimal-length, so consider the other colliders $C$ which are also on $\sinfo^s$. Firstly, each $\sobspaths^{\sssplit}(C)$ is identical to $\sobspaths^s(C)$. Secondly, the split doesn't introduce shorter-length such paths, since any path from $C$ to $D^{\sssplit}$ via some newly added $(X, ``\scopy")$ would correspond to a path via $X^s$ in $\sG^1$ that is at least as short, using the fact that this ID transformation is homomorphic and hence doesn't introduce extra edges. 

This shows that $T^2$ is a tree of systems. Finally, we show that the root infolinks are equivalent:
The if $s^i$ is the root system of ${T^1}$ then $\sssplit^i$ is the root system of $T^2$, and since the infolink of each system in $T^2$ equals that of the corresponding system in ${T^1}$ by definition, we have $\hsecond(X^{\sssplitroot})=X^{\ssRoot}$ and $\hsecond(D^{\sssplitroot})=D^{\ssRoot}$, showing the result.
\end{proof}

\begin{lemma}[Transformation 2 ensures \newAppropriateness]
\label{le:21jan28.4-third-split-ensures-new-appropriateness}
Let $(\sG^1,T^1)$ be any soluble ID graph with complete tree satisfying property (a) \systemsAndPathsUniqueness. Then $(\sG^2,T^2)=\splitsecond(\sG^1,T^1)$ is an ID graph with complete tree satisfying (a) and also (b) \newAppropriateness.
\end{lemma}

\begin{proof} 
We first show that the split preserves fullness, \systemsAndPathsUniqueness and \noOverlappingQX, then we show that it ensures \newAppropriateness:

    {(i)} {(full tree)}. We show that if ${T^1}$ is a full tree, so is $T^2$: For any link $X\to D$ on a path in a system $\sssplit$, that infolink was also on the same path in system $\sssplit$, since $\sssplit$ is identical to $s$ except for the first link on $\sinfo^{\sssplit}$ but that link is in the back section and hence cannot contain $X$ or $D$ by \autosubref{le:20Nov7.1-Basic-properties-of-systems-with-sufficient-recall}{le:20Nov7.1a-no-decisions-in-the-back-section}. Hence there is a system $s'$ in ${T^1}$ with $X\to D$ as its infolink. Hence since by \autoref{def:21jan28.1-third-split-to-ensure-new-appropriateness} the infolink of $\sssplit'$ is the same as that of $s'$, there is a system in $T^2$ that has $X\to D$ as its infolink, namely $\sssplit'$.~\looseness=-1
    
    {(ii)} {(a-\systemsAndPathsUniqueness)}.
    We show that if ${T^1}$ satisfies \systemsAndPathsUniqueness, then so does $T^2$. We state the argument informally: The nodes in $T^2$ are identical to those in ${T^1}$, except for sometimes a split of $X^s$. In that case, $(X^s, ``\scopy")$ is a new node that only appears in system $\sssplit$. Moreover, any other nodes are precisely in system $\sssplit$ if they were in system $s$, so since the node satisfied the required property in ${T^1}$, it also does so in $T^2$.

    {
    (iii) {(b-no-backdoor-infopaths)} Take any system $\sssplit\in \calS^2$. 
    If $s$ is frontdoor info, then $\sssplit$ is identical, so is also frontdoor info.
    If $s$ is backdoor-info, then $\sssplit$ is modified to be frontdoor-info.
    }
\end{proof}

\subsubsection{Transformation 3 (pruning): ensuring \noRedundantLinks}

\begin{definition}[Transformation 3]
\label{def:transformation-4}
Let $\sG^2\!=\!(\sV^2,E^2)$ be an ID graph with tree $T^2$.
Then $\splitthird(\sG^2,T^2)=(\sG^3, T^3)$ 
where $\sG^3=(\sV^3,E^3)$ and the identity homomorphism $\hthird$ are obtained from $\sG^2$ using \autoref{le:21may19.2-CID-hom-from-edge-pruning} by removing 
all \autosubref{def:sep5.3-normal-form-graph-with-tree}{def:sep5.3e-no-redundant-links} (\noRedundantLinks)
links are removed (which are all into non-decision nodes), and where
$T^3=T^2$.~\looseness=-1
\end{definition}

\begin{lemma}[Transformation 3 is homomorphic]
\label{le:transformation-4-is-homomorphic}
$\hthird$ from \autoref{def:transformation-4} is an ID homomorphism from $\sG^3$ to $\sG^2$.
\end{lemma}

\begin{lemma}[Transformation 3 preserves tree]\label{le:transformation-4-preserves-tree}
Let $(\sG^3, T^3)=\splitthird(\sG^2,T^2)$. If $T^2$ is a tree of systems on $\sG^2$ with root link $X\to D$, then $T^3$ is a tree of systems on $\sG^3$ with root link $X'\to D'$ with $\hthird(X')=X$ and $\hthird(D')=D$.
\end{lemma}

\begin{proof}
The tree $T^3$ is rooted at $X \to D$ 
such that $\hthird(X) \to \hthird(D)$
because it is unchanged from $T^2$.
$T^3$ is 
a tree of systems because it is unchanged from $T^2$, 
while $\sG^3$ retains every edge in any path of every system of $T^2$
--- only redundant links 
\autosubref{def:sep5.3-normal-form-graph-with-tree}{def:sep5.3e-no-redundant-links} (\noRedundantLinks)
are removed.~\looseness=-1
\end{proof}

\begin{lemma}[Transformation 3 preserves (a,b) and ensures (c)]
\label{le:transformation-4-ensures-no-redundant-links}
Let $(\sG^2,T^2)$ be any soluble ID graph with complete tree satisfying properties (a,b) of normal form trees. Then $(\sG^3,T^3)=\splitthird(\sG^2,T^2)$ is an ID graph with complete tree satisfying (a-c).
\end{lemma}

\begin{proof} 
The tree $T^3$ on ID graph $\sG^3$ satisfies (b) 
because $T^3\!=\!T^2$ and $T^2$ satisfies (b).
It satisfies (a) because $T^3\!=\!T^2$ and $\sG^3$ has the same set of nodes, and a subset of the edges of $\sG^2$. \ryan{Explain this.}
It satisfies (c) by definition.~\looseness=-1
\end{proof}

\subsubsection{Composing the transformations to obtain an ID graph with normal form tree}

We will now perform these three transformations in order 
to obtain a normal form tree.

\lenormalformtransform*

\begin{proof}

Given that the minimal $d$-reduction $\sG^*$ of $\sG$ contains $X \to D$, we can first pick an arbitrary full tree from \autoref{le:m3.1-existence-of-complete-tree-of-systems} to obtain a tree $T^0$ satisfying (a) \systemsAndPathsUniqueness.

Then, let $(\sG', T')\!=(\calG^3,T^3)=\!\splitthird \circ \splitsecond \circ \splitfirst (\sG^0,T^0)$
and let $h=\!\hfirst \circ \hsecond\circ \hthird$
using \autoref{def:21jan25.2-first-split-to-ensure-systems-and-paths-uniqueness},  \autoref{def:21jan28.1-third-split-to-ensure-new-appropriateness}
and \autoref{def:transformation-4}.
We show that these have each of the desired properties.%

Firstly, $T'$ is normal form:  Each transformation results in a tree with one more property of normal form trees, and preserves the properties of the previous transformations (\autoref{le:21jan26.1-first-split-ensures-systems-and-paths-uniqueness},  \autoref{le:21jan28.4-third-split-ensures-new-appropriateness}, \autoref{le:transformation-4-ensures-no-redundant-links}).~\looseness=-1

Secondly, $h$ is a homomorphism from $\sG'$ to $\sG^0$ since ID homomorphism is preserved under composition (\autoref{le:20dec7.1-composition-of-CID-splits}). 

Thirdly, $\sG'$ is soluble since that is preserved under ID homomorphisms (\autoref{20nov25.1-CID-homomorphism-preserves-sufficient-recall-SR}{}).

Fourthly, each transformation outputs a tree $T^i$ where $h^{(i-1) \gets i}$ maps nodes in the root infolink 
to nodes of infolink of $T^{i-1}$
(\autoref{le:21jan25.6-first-split-preserves-tree},  \autoref{le:21jan28.3-third-split-preserves-tree}
\autoref{le:transformation-4-preserves-tree}),
so the composition has $h(X')=X^0$ and $h(D')=D^0$.

Finally, Transformation 1 results in an ID graph with tree where the nodes in the tree that are also in the original ID graph $\sG^0$ are precisely $X$ and $D$. And transformations 2-4 only remove and add nodes that are not in $\sG^0$, so the property also holds for $G'$ and $T'$, showing the result.
\end{proof}

\section{Value of Information criterion completeness} \label{sec:model-definition}

In \autoref{sec:preliminaries-systems-and-trees}
we have shown that if an infolink $X \to D$ is present in the minimal $d$-reduction of 
a soluble ID graph $\calG$, 
then we can choose a graph and tree $\calG^3, T^3$
so that $\calG^3$ is homomorphic to $\calG$, and $T^3$ is in normal form.
In this section, we will prove that 
we can use $T^3$ to parameterise $\calG^3$
so that optimal performance can only be 
achieved by a policy that has
$\pi^D(\pa)=f(x)$
for a specific $f$
given every $\pa\in \dom(\Pa(D))$ with $P(\pa)>0$.

\subsection{Constructing an ID on nodes in a normal form tree}
We will define an ID for only the nodes in the tree, excluding the root info link, assuming that there is already an ID (possibly trivial) defined for all the other nodes (including those in the root infolink). 
This result is more general than is needed to prove positive VoI (wherein we will chose a trivial ID)
but this is done in order to help with generalizing to the $\sTaskify$ construction in the next section.~\looseness=-1

\newcommand{\sGoriginalInModelConstruction}{\sG^0}
\newcommand{\sGNewInModelConstruction}{\sG^3}
\newcommand{\MoriginalInModelConstruction}{M^0}
\newcommand{\MnewInModelConstruction}{M^3}

\begin{definition}[Parameterization of a normal form tree] \label{def:model-construction}
Let $\sGNewInModelConstruction$ be a soluble ID graph together with a normal form tree of systems $T^3$ with root info-link $X\to D$. Let $\sGoriginalInModelConstruction$ be the subgraph consisting of $X,D$, and all nodes in $\sGNewInModelConstruction$ that are not in $T^3$.
Let $\MoriginalInModelConstruction=(\sGoriginalInModelConstruction, \dom^0, P)$ be an ID on on $\sGoriginalInModelConstruction$, and let $\piTask^{D}$ (which we call \emph{the task for $D$}) be a deterministic decision rule for $D$ that depends only on $X$. 
Then we define the ID $\MnewInModelConstruction=(\sGNewInModelConstruction, \dom, P)$, which are defined as follows:

\begin{itemize}
    \item
    For each node $N$ in $\sGoriginalInModelConstruction$ except $D$, let $\dom(N)=\dom^0(N)$, and let 
    $\dom(D)=\dom_{\sbase}(D)\times \bool$, where
    $\dom_{\sbase}(D)=\dom^0(D)$. 
    For any other chance or decision node $N$, we define the domain of a node by recursion on the tree $T^3$. Let $s$ be the base system of $N$.\footnote{This uses the assumption that $T^3$ is normal form, and hence satisfies \autosubref{def:sep5.3-normal-form-graph-with-tree}{def:sep5.3c-unique-systems-and-paths} (\systemsAndPathsUniqueness), and using the properties that this implies by \autoref{le:21-may-9-position-in-tree-uniqueness-properties}} 
    Assume that the domains of the info node and decision node $X^s, D^s$ of $N$'s base system $s$ are already defined.\footnote{This is well-founded recursion, and for the base case of $s=\sRoot^T$, the domains of $X^s=X$ and $D^s=D$ were already defined above.} 
    Then, if $s$ is of the non-directed-info case then
        \[\dom_\sbase(N) =
            \begin{cases} 
            \bool \casesif  {$N$ is on $\sobspaths^s(C^1)$, incl. $C^1$, the first obspath of $s$}\\
            \dom{(X^s)} \casesif  {$N$ is in between $X^s\pathto C^1$ on $\sinfo^s$ }\\
            \bool^{|\dom_\sbase{(D^s)}|} \casesif  {$N$ in any other part of  $\sinfo^s$, or any other $\sobspaths^s$}\\
            \dom(D^s) \casesif {$N$ in $\scontrol^s$}
            \end{cases},\]
    and if it is of the directed-info case then $\dom_\sbase(N) = \dom(X^s)$. 
    Based on this:
        \[\dom(N) =
            \begin{cases} 
            \dom_\sbase(N) \times \bool \casesif  {$N=D^{s'}$ for a non-directed-info descendant  $s'$}\\
            \dom_\sbase(N) \casesotherwise
            \end{cases}.\]
    \ryan{Relatedly, can we just refer to dom-base with $dom(Q^s)$ or something, or do we need this term?}
    \chris{We might be able to get rid of it, but I'm not immediately sure how.}
    
    \ryan{It seems like the domains are implied by  the functions and domains of the parents in all cases other than forks, and the root systems so can we just define the functions directly?}
    \chris{We maybe should consider this, but it's probably not a priority. It maybe has some downsides in terms of understandability, but my guess is it'd be better if it can be done clearly and rigorously. Probably the main problem is that decision nodes need domains too.}

    \item For each chance node $N$ in $\sGoriginalInModelConstruction$ (including $X=X^{s^\sRoot}$), let $P_{\MnewInModelConstruction}^N=P^N_{\MoriginalInModelConstruction}$. For any other chance node $N$, let $s$ be the base system of $N$ and $p$ the base path of $N$,\footnote{This uses the assumption that $T^3$ is normal form, and thus satisfies (c) \systemsAndPathsUniqueness} and let $\piTask^{D^s}$ be the task of the decision of system $s$, defined for $s=\sRoot$ as the task $\piTask^{D}$ given above, and as the identity operation \emph{$\id^{\spath}$}
    for every other system. Then, writing $\piTask^{D^s}(x)$ to refer to $d$ s.t. $\piTask^{D^s}(d|x)=1$, we let $P^N(n|\pa)=1$ iff $n=f^N(\pa)$, where
    \[f^N =
        \begin{cases} 
        {\id}^{\spath} \casesif {$\to N\to$ or  $\leftarrow N \leftarrow$ in $\spath$}\\
        N^2[\piTask^{D^s}(N^1)] \casesif {$N^1\to N\gets N^3$ is the first collider on $\sinfo^s$}\\
        {\XORtext}^{\spath}\casesif {$\to N \leftarrow$ is a collider on $p$ other than the first}\\
        \srandom^\spath \casesif {$\leftarrow N \to$ in $\spath$}
        \end{cases}
    \]
    where $\id^{\spath}$ copies the output of the $\dom_\sbase$ part of the previous node on $\spath$, $\XORtext^{\spath}$ takes the bitwise \emph{exclusive OR} operation on bitstrings,\footnote{Note that the domain of any such collider is a bitstring.} $\srandom$ copies a uniform random bitstring from $\Epsilon^N$, and \emph{$b[x]$} outputs the $x^\text{th}$ digit of a bitstring $b\in \bool^n$. 
    
    \item For each utility node $U$ in $\sGoriginalInModelConstruction$, let $f_{\MnewInModelConstruction}^U=f^U_{\MoriginalInModelConstruction}$. For any other utility node $U$,
    if $\sinfo^s$ is non-directed for the base system $s$ of $U$, letting $i$ be the value that $U$ receives from the penultimate node in the infopath and $(c^h,c^r)$ the value it receives from the penultimate node of the control path, we let $P^U(u|\pa)=1$ iff $u=f^U(\pa)$, where
    \begin{align*}
    f^U(i,(c^h,c^r)) &=
        \begin{cases} 
            \Umax \casesif {$c^h[i]=c^r$}\\
            0 \casesotherwise
        \end{cases},
    \end{align*}
    and if instead $\sinfo^s$ is directed then
    \begin{align*}
    f^U(i,c) &=
        \begin{cases} 
            \Umax \casesif {$i=c$}\\
            0 \casesotherwise
        \end{cases},
    \end{align*}
    where $\Umax=1+\sum\limits_{U\in \sU^{\Goriginal}}\left(\max\limits_{\pa} f^U_{\MoriginalInModelConstruction}(\pa)-\min\limits_{\pa} f^U_{\MoriginalInModelConstruction}(\pa)\right)$.
\end{itemize}
\end{definition}

\subsection{The ID forces decision nodes to ``perform their task"} \label{sec:task-performance}

We will show here that we have constructed the ID in such a way that a decision $D^s$ of a system $s$ can only achieve optimal utility if it performs its task. Recall that we write $\piTask^{D^s}(x)=d$ to refer to $d$ s.t. $\piTask^{D^s}(d|\pa)=1$ where $x$ is the value of $X^s$ under $\pa$.

\begin{definition}[Task performance] \label{def:aug23.4-task-performance}
Let $s$ be a system in some tree $T$.
Given a decision context $\pa\in \dom(\Pa(D^s))$, we say that $D^s$ \textit{performs the task $\piTask^{D^s}$ with $\pi$ given $\pa$} if 
\begin{itemize}    
    \item $\sinfo^s$ is directed, and $\pi^{D^s}(d|\pa)=\piTask^{D^s}(d|\pa)$; or
    
    \item $\sinfo^s$ is non-directed, and $\pi^{D^s}(d|\pa)=1$ iff $d=(\piTask^{D^s}(\pa), \hat R)$, where $\hat R$ is the value of $Q^s[\piTask^{D^s}(X^s)]$.
\end{itemize}
And we say that \emph{$D^s$ performs the task $\piTask^{D^s}$ with $\pi$} if it does so for any $\pa$ that has positive probability of occurring under $\pi$. 
\end{definition}

\begin{lemma}[\editingMode{le:3.3 - }\modelKnowledgeLemmaName]\label{le:3.3-model-knowledge-lemma}
Let $M^3$ be the ID based on some $T^3$, $M^0$ and $\piTask^{D}$ (\autoref{def:model-construction}), and let $\piTask^{D^s}$ be the task of system $s$ in $M^3$, as defined in  \autoref{def:model-construction}.\footnote{recall that the task of all decisions in the tree except for the root decision, are the identity operations.}
Let $s$ be a system with non-directed infopath, and assume that for all child systems $s'$ of $s$, $D^{s'}$ performs its task $\piTask^{D^{s'}}$ with $\pi$.
Then $Q^s[\piTask^{D^s}(X^s)]$ can be expressed as a function of $\Pa(D^s)$
but $P(Q^s[\tilde y]=1\mid \Pa(D^s)=\pa)=\frac{1}{2}$ for all $\pa\in \dom (\Pa(D^s))$
for any $\tilde{y} \neq \piTask^{D^s}[X^s]$.
\end{lemma}

\begin{proof} 
Let $x^s$ be any value in $\dom(X^s)$, and for brevity, let $y=\piTask^{D^s}(x^s)$.\chris{should be $\piTask^{D^s}(X^s)$ I think? or something like that, bc now its quantifying over $x$} Moreover, let $\xor$ denote the exclusive or operator (XOR) on boolean strings.
We will first show that (1) $Q^s[y]$ is a function of $\Pa(D^s)$, 
and then that (2) $P(Q^s[\tilde y] \; \mid \;\Pa(D^s))$ is uniformly random for $\tilde y \in \dom_{\sbase}(D^s) \setminus \{y\}$.

First we show (1).
First note that $Q^s$ and each fork $F^i,1\leq i\leq k$ in the infopath is a random bitstring of length $\bool^{|\dom_\sbase{(D^s)}|}$ (\autoref{def:model-construction}). And since each obsnode $O^i$ equals the collider $C^i$ on the info path (where $O^i$ is the obs node of the obs path $\sobspaths^s(C^i)$), by the same construction, $O^1=F^1[\piTask^{D^s}(x)]$, and for $1<i< k$, $O^i=F^{i-1} \xor F^i$, and $O^k=F^k \xor Q^s$.
Then, the decision can recover $Q^s[y]$ by taking the XOR of the $y^{\text{th}}$ element of each observation, by setting (letting $x$ be the value of $X^s$ in $\pa$):~\looseness=-1
\begin{alignat*}{2}
    &\;\pi^D(\pa, \varepsilon)\\
    &:= \;O^1 \xor O^2[y]\;\xor\;...\;\xor\; O^{k-1}[y] \xor O^k[y] && \quad : \text{where $y=\piTask^{D^s}(x)$}\\
    &=\; F^1[y] \xor (F^1[y]\xor F^2[y])\;\xor \;... \;\xor\; 
    (F^k[y] \xor Q^s[y]) &&\quad : \text{by \autoref{def:model-construction}} \\
    &=\; (F^1[y] \xor F^1[y] )\;\xor \;... \;\xor\; (F^k[y] \xor F^k[y] )\xor Q^s[y] &&\quad : \text{associativity of $\xor$} \\
    &=\; Q^s[y] &&\quad : \text{$b\xor b=0$ and $0\xor b=b$}
\end{alignat*}

Now we will show (2). Firstly, \autoref{le:2v2-graph-knowledge-lemma} directly implies that $(\Pa(D^s)\setminus (\sV^s\cup \sO)$ are $d$-separated from $Q^s$ conditional on $(\Pa(D^s)\cap \sV^s)\cup \sO$.  Hence by the standard $d$-separation criterion, they are probabilistically independent as well, i.e. $P(Q^s[\tilde{y}]=b \mid \pa(D^s), \sO)=P(Q^s[\tilde{y}]=b \mid \pa(D^s)\cap \sV^s, \sO)$. Moreover, since by
\autosubref{le:20dec14.1-basic-properties-of-trees-with-SR}{le:20dec14.1b-only-info-links-from-observation-nodes-to-ancestor-decisions} only observation nodes in this subtree can be parents, and since conditioning on more information does not increase uncertainty, it will suffice to show that
$P(Q^s[\tilde y] = b \; \mid \; X^s=x^s, \sO=\so)= \frac 1 2$ for all $b, x^s, \so$.

Now we will show (2). 
For notational brevity, we define:
$\sQ=\bigtimes_{s \in \Desc_s} Q^s$,
$\sF=\bigtimes_{s \in \Desc_s,1\leq i\leq \mid \sF^s \mid} F^{s,i}$,
$\sO=\bigtimes_{s \in \Desc_s,1\leq i\leq \mid \sO^s \mid} O^{s,i}$.\chris{sentence below is not exactly clear what the proposition is. }
Then, using \autoref{le:2v2-graph-knowledge-lemma}, we can express the probability of $Q^s[\tilde{y}]$ as a 
probability, conditional on observation nodes of the subtree.
\begin{align*}
&P(Q^s[\tilde{y}]=b \mid \pa(D^s)) \\
&= \sum_{\so \setminus \Pa(D^s)} P(\so \setminus \Pa(D^s)) \cdot P(b \mid \so \setminus \Pa(D^s), \pa(D^s))  && : \text{product rule} \\
&= \sum_{\so \setminus \Pa(D^s)} P(\so \setminus \Pa(D^s)) \cdot P(b \mid \so \setminus \Pa^s, \pa^s, \pa^{-s} \setminus \sO) && : \text{regroup terms in conditional} \\
&= \sum_{\so \setminus \Pa(D^s)} P(\so \setminus \Pa(D^s)) \cdot P(b \mid \so \setminus \Pa^s, \pa^s)  && :\autoref{le:2v2-graph-knowledge-lemma} \\
&= \sum_{\so \setminus \Pa(D^s)} P(\so \setminus \Pa(D^s)) \cdot P(b \mid x^s, \so))  && (\dagger)
\end{align*}

So it will suffice to show that 
$P(Q^s[\tilde y] = b \; \mid \; X^s=x^s, \sO=\so)= \frac 1 2$ for all $b, x^s, \so$.
Let $\mathfrak{X}^{\so,b,x^s}_{\sQ \cup \sF} = \{\sq,\sf : P(Q^s[\tilde{y}]=b,\so \mid \sq,\sf,x^s)=1\}$.
In the ID $\calM^3$, the event $Q^s[\tilde y] = b,\sO=\so$ may be equivalently stated as follows:

\begin{enumerate}
    \item $Q^s[\tilde{y}]=b $
    \item $f^{s,1}[\piTask^{D^s}[x^s]]=o^{s,1} (\iff O^1=o^{s,1})$
    \item $f^{s',1}[\piTask^{D^s}[x^{s'}]=o^{s',1} \text{ for } s' \neq s (\iff O^{s',1}= \so^{s',1})$
    \item $f^{s',i} \xor f^{s',i+1}=o^{s,i} \text{ for } 1 < i \leq \mid F^{s'}\mid \text{ for } s' \in \calS (\iff \sO^{s',i} = \so^{s',i},i>1)$
\end{enumerate}

That is to say that 
$\mathfrak{X}^{\so,b,x^s}_{\sQ \cup \sF}= \{(\sq,\sf) \in \dom(\sQ \cup \sF):\text{s.t. (1-4) satisfied}\}$, 
and that for $(\sq,\sf) \not \in \mathfrak{X}^{\so,b,x^s}_{\sQ \cup \sF}$, 
$P(Q^s[\tilde{y}]=b,\so \mid \sq,\sf,x^s)=0$.

We can use the definition of $\calM^3$ to yield a convenient equivalent expression for 
$\mathfrak{X}^{\so,b,x^s}_{\sQ \cup \sF}$ that substitutes $Q^s$ in for $F^{s,1}$.

\begin{enumerate}[label={\arabic*$'$.}] \setcounter{enumi}{1}
\item $(Q^s[\piTask^{D^s}(x^s)]=o^{s,1}\xor \bigoplus_{1 < i \leq \mid F^s\mid} o^{s,i}) \Leftarrow (2.,4.)$
\item $q^{s'}[\piTask^{D^{s'}}(X^{s'})]=o^{s',1} \xor \bigoplus_{1 < i \leq \mid F^{s'}\mid} \Leftarrow (3.,4.)$.
\end{enumerate}
Since XOR with a bitstring is a bijective operation, we also have: $(2'.,4.) \Rightarrow 2$ and $(3'.,4.) \Rightarrow 3$.
Thus $\mathfrak{X}^{\so,b,x^s}_{\sQ \cup \sF}= \{(\sq,\sf) \in \dom(\sQ \cup \sF):\text{s.t. (1.,2',3',4) satisfied}\}$

To obtain a $(\sq,\sf) \in \mathfrak{X}^{\so,b,x^s}_{\sQ \cup \sF}$, one can then carry out the following algorithm starting from $s^0=s$:
\begin{enumerate}[label=\alph*)]
\item choose any $q^s$ that satisfies (1,2') (there are $\mid Q^s \mid / 4$ possible assignments because $\piTask^{D^s}[x^s] \neq \tilde y$)
\item choose each $F^{s,i}$ to satisfy (4) given $q^s$ (one possible assignment)
\item for each child system $s^j$, choose $Q^{s^j}$ to satisfy (3') given $q^s$ ($\mid Q^{s^j} \mid / 2$ possible assignments).
\item choose each $F^{s^j,i}$ to satisfy (4) given $q^{s^j}$ (one possible assignment)
\item repeat (c-e) for children of $s^j$
\end{enumerate}

Any $(\sq,\sf)$ computed by this algorithm will clearly satisfy (1.,2',3',4) 
and is therefore in $\mathfrak{X}^{\so,b,x^s}_{\sQ \cup \sF}$.
Conversely, if at any step, a non-allowed assignment is selected, then one of (1.,2',3',4) 
is violated, and so $(\sq,\sf) \not \in \mathfrak{X}^{\so,b,x^s}_{\sQ \cup \sF}$.
The number of possible assignments after (a-b) (is clearly $\mid Q^s \mid / 4=\mid Q^{s^0} \mid / 2^{j+2}$.
For the $j$th system, the number of possible assignments is then 
$\mid Q^{s^0}/4 \cdot Q^{s^1}/2 \ldots Q^{s^j}/2 = \bigtimes_{j'\leq j} \frac{\mid Q^{s^{j'}}\mid}{2^{\mid 2^j}}$.
So $\mid \mathfrak{X}^{\so,b,x^s}_{\sQ \cup \sF} \mid = \bigtimes_{j'\leq \mid \calS \mid} \frac{\mid Q^{s^{j'}}\mid}{2^{2+j}}$

This is independent of the value $b$ selected for $Q^s[\tilde{y}]$, so 
for any $\so,b,x^s, \tilde{y}\neq \piTask^{D^s}(x^s)$:
\begin{align*}
    P(Q^s[\tilde{y}]=b,\so \mid x^s) &= \sum_{\sf \in \dom(\sF),\sq \in \dom (\sQ)} p(Q^s[\tilde{y}]=b,\so\mid \sq,\sf, x^s)p(\sq,\sf \mid x^s) \\
    &= \sum_{(\sf,\sq) \in \mathfrak{X}^{\so,b,x^s}_{\sQ \cup \sF}} p(Q^s[\tilde{y}]=b,\so\mid \sq,\sf, x^s)p(\sq,\sf \mid x^s) \\
    &= \mid \mathfrak{X}^{\so,b,x^s}_{\sQ \cup \sF}\mid p(\sq,\sf \mid x^s) & (*)
\end{align*}\chris{add better explanation for last step}

Thus we can compute the conditional:
\begin{align*}
P(Q^s[\tilde{y}]=b \mid \so, x^s)
&= \frac{P(Q^s[\tilde{y}=b],\so \mid x^s)}{P(\so \mid x^s)} & \\
&= \frac{P(Q^s[\tilde{y}=b],\so \mid x^s)}{\sum_{b \in \bool} P(Q^s[\tilde{y}]=b,\so \mid x^s)}& \\
&= \frac{\mid \mathfrak{X}^{\so,b,x^s}_{\sQ \cup \sF}\mid p(\sq,\sf \mid x^s)}{(\mid \mathfrak{X}^{\so,0,x^s}_{\sQ \cup \sF}\mid + \mid \mathfrak{X}^{\so,1,x^s}_{\sQ \cup \sF}\mid ) p(\sq,\sf \mid x^s)}&&:\text{substituting in from (*)}& \\
&= \frac{1}{2} && :\mathfrak{X}^{\so,b,x^s}\text{ is constant in $b,x^s$} %
\end{align*}
Substituting into $(\dagger)$ yields the result.
\end{proof}

We now show that in the ID that we constructed in \autoref{def:model-construction}, the set of optimal policies are exactly the set of policies in which all decisions perform their task.

\begin{definition} [Locally optimal decision rule] \label{def:21feb6.2-local-optimality}
The decision rule $\pi^D$
is locally optimal with respect to $\spi^{-D}=\{\pi^{D'}\}_{D' \neq D}$
if the policy $\pi^D \cup \spi^{-D}$
has greater or equal expected utility
as the policy $\tilde{\pi}^D \cup \spi^{-D}$ 
obtained by any alternative decision rule $\tilde \pi^D$.
\end{definition}

\begin{lemma}[Decisions in the tree are optimal iff they perform their task] \label{le:1.7-optimal-iff-truthful}
Let $\pi$ be any policy on the taskified ID $M^3$ obtained from some $T^3$, $M^0$, and task $\piTask^{D}$, and let $\piTask^{D^s}$ be the task of system $s$ in $M^3$, as defined in  \autoref{def:model-construction}. Then the decisions $D^s$ in the tree $T$ are optimal under $\pi$ if and only if they all perform their task $\piTask^{D^s}$ with $\pi$.  Moreover, then $\EE_\pi(U^s)=\Umax$ for all $s$.
\end{lemma}

\begin{proof}
We first prove an intermediate step, to be used for induction to show the result. It essentially says that as long as the decisions of the child systems of $s$ perform their task, then so will $D^s$:

(Induction-step) : Assume $\pi$ is a policy such that for any child system $s'$ of $s$ in $T$, the decision $D^{s'}$ performs its task with $\pi$. Then for any $\pa\in \Pa(D^s)$ with positive probability of occuring under $\pi$, it holds that $U^s=\Umax$ with probability $1$ if $D^s$ performs its task with $\pi$ given $\pa$, and $U^s=0$ with probability $1$ if it outputs any other decision. Moreover, such a policy $\pi$ exists.

We show (Induction-step) as follows: Since by assumption all the decisions on the info and control paths perform their task, they copy the previous node's value. The chance nodes in the front section do the same by definition of the ID.
    Hence for any $\pa\in \dom(\Pa(D^s))$ that occurs with positive probability, we have that with probability $1$, $U^s$ receives from the penultimate node on the control path a value that 
    equals that of $D^s$, and from the penultimate node in the info path a value that, if $s$ is of the directed-info case equals that of $X^s$, and if it is of the non-directed-info case equals that of $Q^s$.
    
    Hence if we assume that $s$ is of the directed-info case, then substituting this into the definition of $f^{U^s}$, and using \autoref{def:aug23.4-task-performance} (task performance), this implies that conditional on $\pa$, $\Umax$ is attained with probability $1$ if and only if $D^s$ copies:
        \[U^s = 
        \begin{cases}
        \Umax \casesif {$D^s=X^s$}\\
        0  \casesotherwise 
        \end{cases}.\]
    
    This obviously implies $\EE_\pi[U^s]=\Umax$. Similarly, if we assume that $s$ is of the non-directed-info case, then substituting the same into the definition of $f^{U^s}$, and using \autoref{def:aug23.4-task-performance} (task performance), this implies that conditional on $\pa$, $\Umax$ is attained with probability $1$ if and only if $D^s$ reports consistently: Letting
    $D^s= (\hat T,\hat R)$:
    \[U^s = 
        \begin{cases}
        \Umax \casesif {$Q^s[\hat T]=\hat R$}\\
        0  \casesotherwise 
        \end{cases}.\]
    Hence $\EE_\pi[U^s]=\Umax$ if and only if $D^s$ reports consistently with probability $1$ conditional on its parents. Moreover, by \autoref{le:3.3-model-knowledge-lemma} the value of $Q^s$ is known conditional on its parents only for the true value of the info node $X^s$. Hence $\EE_\pi[U^s]=\Umax$ if and only if $D^s$ reports consistently and outputs $\hat X=X^s$. The latter is equivalent to $D^s$ performing its task, and hence also shows that it is possible for $D^s$ to perform its task. This shows (Induction-step).

    Now we prove the $\Leftarrow$ direction of the main result: Let $\pi$ be any policy so that all decisions $D$ perform their task with $\pi$. Then the $\Leftarrow$ direction of (Induction-Step) implies that for any system $s$, $U^s=\Umax$ with probability $1$ given $\pa$, and since $\Umax$ is defined such that it is larger than the sum of the possible range of utility for all utility nodes in the original graph, by not performing its task any such $D^s$ would decrease $U^s$ by $\Umax$ and increase the other utility nodes by at most $\Umax -1$.
    Since such a policy exists by (Induction-step), this also shows that optimal policies achieve $\EE_\pi(U^s)=\Umax$ for all $s$.~\looseness=-1
    
    We now  prove the $\Rightarrow$ direction of the main result: Assume all decisions $D^s$ in the tree are optimal under $\pi$. We prove the statement by backward induction on decisions using (Induction-Step), where we show the base step by applying the induction step to the final decision: As the induction hypothesis, assume that all $D^{s'}$ with $D^{s'}>D^s$ 
    perform their task. Then since decisions of descendant systems are descendants (\autosubref{le:20dec14.1-basic-properties-of-trees-with-SR}{le:20dec14.1a-decisions-in-descendant-systems-are-descendants}), all the decisions of child systems of $D^s$ perform their task. This implies by the (Induction-Step) that for any $\pa$ with positive probability of occurring under $\pi$, 
    $U^s=\Umax$ with probability $1$ given $\pa$ if $D^s$ performs its task with $\pi$ given $\pa$, and $U^s=0$ if it takes any other decision. And since $\Umax$ is defined such that it is larger than the sum of the possible range of utility for all utility nodes in the original graph, by not performing its task any such $D^s$ would decrease $U^s$ by $\Umax$ and increase the other utility nodes by at most $\Umax -1$. Hence since by assumption $D^s$ is optimal with $\pi$ it must perform its task with $\pi$ given all $\pa$ that occur with positive probability under $\pi$.~\looseness=-1
\end{proof}

\subsection{Showing positive VoI on an ID graph that has a normal form tree} \label{sec:model-has-incentives}

In this section we will show that we can use the ID from \autoref{def:model-construction} applied to the info link $X\to D$ to show that $X$ has positive VoI for $D$.
In the previous section we showed that a policy is optimal if and only if it performs its task. We now show that in order to perform its task, $D$ in fact has to observe $X$. These together will be used to show that there is in fact positive VoI for $X$ on $D$ on the ID $\calM$ as desired.~\looseness=-1

\ryan{Need to replace materiality with reference to VoI}
\chris{Isn't this done already?}

Here is the main completeness result for the ID $\calM$:

\materialitysat*

\begin{proof}
Let the normal form tree be $T$.
Then, let the ID $\calM$ be obtained from $\MoriginalInModelConstruction$ and $\piTask^{D}$ and $T$ (\autoref{def:model-construction}), where $\MoriginalInModelConstruction$ is the ID that assigns boolean domains to $X$ and $D$ and trivial domains to all other nodes in $\sGoriginalInModelConstruction$, and has $X$ generate a random bit, and where $\piTask^{D}$ is the identity function.

We know that a policy in $\calM$ is optimal iff it performs its task (\autoref{le:1.7-optimal-iff-truthful}), i.e. for the particular ID $\calM$, to output $\piTask^{D^s}(X)=X$.
Therefore in order to show that $X$ has positive VoI for $D$, it suffices to show that any policy on $\calM$ that performs its task does not factor over $M_{X \not \to D}$, for which it suffices to show that $P[X | \Pa(D) \setminus \{X\}]=\frac{1}{2}$. 

\ryan{Add the missing bit from the previous materiality completeness proof here.}

Next, we will prove that $(\sQ,\sF)$ is independent of $X$. For any $\sq,\sf,x'$:

\begin{align*}
    P(\sq,\sf\mid x') &= P(\sq,\sf) \frac{P(x'\mid \sq,\sf)}{P(x')} && : \text{Bayes Theorem} \\
    &= P(\sq,\sf) \frac{P(x' \mid \doo(\sq,\sf))}{P(x')} && :\text{do-calc rule 2; $\Pa_{\sQ}=\emptyset$ in $\calG^3$} \\
    &= P(\sq,\sf) \frac{P(x'}{P(x')} && :\text{do-calc rule 1; $\Pa_{X}=\emptyset$ in VoI ID} \\
    &= P(\sq,\sf) && (*)
\end{align*}
\ryan{What do we mean by ``VoI ID'' here?}
\begin{align*}
P(x \mid \so) &= P(x) \frac{P(\so\mid x)}{P(\so)} && : \text{Bayes Theorem} \\
 &= P(x) \frac{P(\so\mid x)}{\sum_{x'}P(\so\mid x')P(x')} && : \text{Addition rule} \\
 &= P(x) \frac{\sum_b P(Q^s[\tilde y]=b,\so \mid \sq,\sf,x) P(\sq,\sf \mid x)}{\sum_{x,b} P(Q^s[\tilde y]=b,\so \mid \sq,\sf,x) P(\sq,\sf \mid x)} \text{ for any } \tilde y && : \text{Addition, product rules} \\
&=P(x) \frac{\sum_b \mathfrak{X}^{\so,b,x^s}_{\sQ \cup \sF} P(\sq,\sf \mid x)}{\sum_{x,b}  \mathfrak{X}^{\so,b,x^s}_{\sQ \cup \sF} P(\sq,\sf \mid x)P(x')} && : \text{Property of $\mathfrak{X}^{\so,b,x^s}_{\sQ \cup \sF}$} \\
&=P(x) \frac{P(\sq,\sf)}{\sum_{x'} P(\sq,\sf)P(x')} && :\mathfrak{X}^{\so,b,x^s}\text{constant in $b$; (*)} \\
&=P(x) \frac{P(\sq,\sf)}{P(\sq,\sf)} &&: \sum_{x'} P(x')=1 \\
&=P(x) && :\text{$P(\sq,\sf)>0$} \\
&= \frac{1}{2}
\end{align*}

\end{proof}

\end{document}